\documentclass[letterpaper, 10pt, conference]{ieeeconf}
\usepackage{times}
\usepackage[pdftex]{graphicx}
\usepackage{subfigure}

\usepackage{amsmath,amssymb,amsopn,amstext,amsfonts, amsthm}
\usepackage{cancel}
\usepackage[space]{cite}
\usepackage{pdfsync}
\usepackage{balance}
\usepackage{color}
\usepackage{mathtools}
\usepackage[ruled,norelsize, linesnumbered]{algorithm2e}
\usepackage{bm}
\usepackage{diagbox}
\usepackage{float}
\usepackage[outdir=./]{epstopdf}
\usepackage{url}
\usepackage{multirow}
\usepackage{makecell}
\usepackage{adjustbox}
\usepackage[font=small,skip=2pt]{caption}

\usepackage[linkcolor=black,citecolor=black,urlcolor=black,colorlinks=true]{hyperref}

\newtheorem{prop}{Proposition}

\theoremstyle{remark}

\newcommand{\func}{\mathtt}
\newcommand{\point}{\mathbf}
\newcommand{\set}{\mathcal}

\SetKw{Continue}{continue}
\makeatletter
\newcommand{\removelatexerror}{\let\@latex@error\@gobble}
\makeatother

\DeclareGraphicsExtensions{.pdf,.png,.jpg,.eps}
\IEEEoverridecommandlockouts
\title{\LARGE \bf  Trajectory Replanning for Quadrotors Using \\ Kinodynamic Search and Elastic Optimization }
\author{Wenchao Ding, Wenliang Gao, Kaixuan Wang, and Shaojie Shen%
  \thanks{This work was supported by the Hong Kong PhD Fellowship Scheme, and the Joint PG Program under the HKUST-DJI Joint Innovation Laboratory. All authors are with the Department of Electronic and Computer Engineering,
          Hong Kong University of Science and Technology, Hong Kong, China.
          {\tt\small wdingae@ust.hk, wenliang.gao@ust.hk, kwangap@ust.hk, eeshaojie@ust.hk}}%
}
\begin{document}

\maketitle
\thispagestyle{empty}
\pagestyle{empty}

\begin{abstract}
We focus on a replanning scenario for quadrotors where considering time efficiency, non-static initial state and dynamical feasibility is of great significance. We propose a real-time B-spline based kinodynamic (RBK) search algorithm, which transforms a position-only shortest path search (such as A* and Dijkstra) into an efficient kinodynamic search, by exploring the properties of B-spline parameterization. The RBK search is greedy and produces a dynamically feasible time-parameterized trajectory efficiently, which facilitates non-static initial state of the quadrotor. To cope with the limitation of the greedy search and the discretization induced by a grid structure, we adopt an elastic optimization (EO) approach as a post-optimization process, to refine the control point placement provided by the RBK search. The EO approach finds the optimal control point placement inside an expanded elastic tube which represents the free space, by solving a Quadratically Constrained Quadratic Programming (QCQP) problem. We design a receding horizon replanner based on the local control property of B-spline. A systematic comparison of our method against two state-of-the-art methods is provided. We integrate our replanning system with a monocular vision-based quadrotor and validate our performance onboard.
\end{abstract}

\section{Introduction}
Micro aerial vehicles (MAVs), in particular quadrotors, have gained wide popularity in various inspection and exploration applications. To meet the need for fully autonomous navigation in unexplored environments, a real-time local replanner producing smooth, dynamically feasible trajectories is of great significance. Many existing planning methods \cite{richter2016polyunqp}\cite{mellinger2011minsnap}\cite{liu2017sfc} \cite{Feigao2017hg} follow a path planning and path parameterization two-step pipeline. The path planning part only produces unparameterized path, while the path parameterization part chooses a feasible dynamical profile for the path. The two-step pipeline is popular due to its efficiency, but is problematic in replanning scenario, since the path planning part is unaware of the vehicle's non-static initial states. For instance, the geometrically shortest path may diverge from the initial velocity direction, resulting in jerky trajectories or failure of the path parameterization.
\begin{figure}[htb]
	\vspace{+0.5cm}
	\centering
	\includegraphics[width=0.30\textwidth]{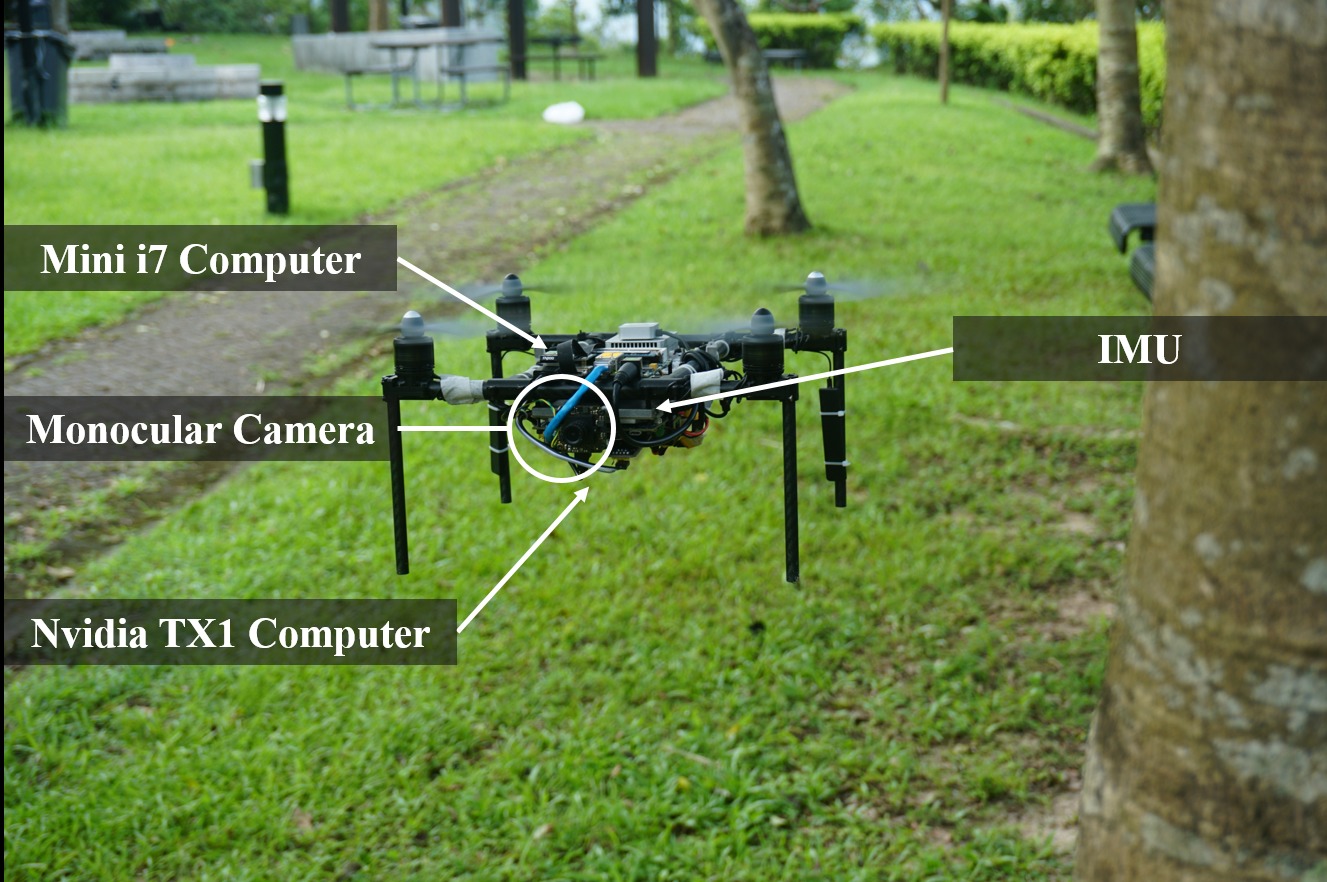}
	\caption{Illustration of our experimental testbed which is equipped a monocular camera, an Intel i7 processor and a NVIDIA Jetson TX1. The localization module is based on our Monocular Visual Inertial Navigation System (VINS-Mono) \cite{qin2017vins}, and mapping module is based on monocular dense mapping and TSDF fusion. \url{https://www.youtube.com/watch?v=obwV1PFuPC0}
		}
	\label{fig:drone}
	\vspace{-0.75cm}
\end{figure}

Considering the replanning for quadrotor which has non-trivial dynamics, it is highly desirable to use kinodynamic motion planners, to facilitate non-static initial state and ensure dynamical feasibility. Sampling-based methods such as kinodynamic RRT* \cite{webb2013kinodynamic} are asymptotically optimal but computationally expensive with an execution time of 10s to 100s. Allen \textit{et al.} \cite{allen2016real} proposed a real-time kinodynamic adaptation of FMT*, however, the computation time is still around half a second. Another issue of sampling-based methods is that the randomized behavior may result in unpredictable performance \cite{liu2017sfc}, especially when only limited sampling is permitted. Search-based method is suitable for replanning in the sense that its results are consistent given similar observations of the environment. To this end, Liu \textit{et al.} \cite{Liu2017smp} proposed a primitive-based resolution-complete (i.e., optimal in the induced lattice graph) search method using linear quadratic minimum time control. However, the running time is sensitive to both the discretization level and the order of control input. For the high-order control input (such as jerk-controlled) or large dynamic range which needs a higher discretization level, the running time is not stable and may be around one or two seconds.

In this paper, we explore a novel angle of kinodynamic search by searching for B-spline control point placement, to improve the time efficiency. Different from many existing works \cite{elbanhawi2015bspsmooth}\cite{koyuncu2008probbsp}\cite{maekawa2010curbspline} which follow the two-step pipeline and use B-spline as the path parameterization method, we produce time-parameterized trajectory in the search process directly. Our contribution is that we explore three properties of B-spline to transform a low-complexity position-only search (such as A* and Dijkstra) into an efficient kinodynamic search. First, thanks to B-spline's \textit{local control} property as elaborated in Sec.~\ref{sec:Bspline_property}, the evaluation of B-spline can be done locally, resulting in an efficient expansion of control points, without re-evaluation of whole trajectory. Second, based on the \textit{convex hull} property of B-spline, the dynamical feasibility constraints can be rewritten in a linear form, avoiding the computation overhead brought by feasibility checking. Third, the control cost can be evaluated in \textit{closed-form}, without solving computationally expensive two-point boundary value problem BVP problem. The resulting real-time B-spline based kinodynamic (RBK) search shares a similar structure and complexity to the position-only search, but outputs dynamically feasible time-parameterized trajectory, which facilitates non-static initial state.

The RBK search is essentially a greedy adaptation of a full-scale B-spline based kinodynamic search. The full-scale search is optimal, but scales poorly with respect to the spline order and number of connections of the grid, as elaborated in Sec.~\ref{sec:kinodynamic_search}. As verified in Sec.~\ref{sec:rbk_performance}, the solution cost (the control cost and time cost) of the RBK seach is between the full-scale kinodynamic search and the position-only shortest path search, and the RBK search is at least two orders of magnitude faster than the full-scale search. To cope with the potential sub-optimality and grid discretization in the RBK search, we propose an elastic optimization (EO) approach as a post-optimization process, which finds the optimal control point placement inside the expanded elastic tube and ensures dynamical feasibility, by solving a Quadratically Constrained Quadratic Programming (QCQP) problem. Thanks to the enforcement of dynamical feasibility during the RBK search, the initial trajectory fed to elastic optimization is already feasible, enhancing the robustness of the post-optimization.

We complete the replanning system by further adopting a receding horizon planning (RHP) framework \cite{liu2017sfc} with a B-spline specification using its local control property. A sliding window optimization strategy\cite{usenko2017bspgradient} is also introduced to bound the complexity of EO approach. We compare our replanning system with two state-of-the-art replanners \cite{Liu2017smp}\cite{oleynikova2016ct} in simulation, and validate our performance on a real monocular vision-based quadrotor as shown in Fig.~\ref{fig:drone}.

\section{System Overview}
\label{sec:overview}
The overview of our autonomous planning system is shown in Fig.~\ref{fig:system overview}. We call RBK search the system \textit{front-end}, and EO approach \textit{back-end}. The RBK search provides an initial placement of control point, which will be refined during the optimization process. The elastic optimization approach consists of two steps, namely, the elastic tube expansion and elastic optimization. There are two modes of our front-end, i.e., passive mode and active mode. For the passive mode, the front-end is activated by collision checking, while in the active mode, the front-end is not only activated by collision checking, but also by a timer depending on the B-spline knot separation. The active mode is used to reject the effects of outliers, when the mapping is noisy in some circumstances. Basically, we work on a series of densely connected B-spline control points, which are constantly modified by the RBK search to avoid collision, and refined by the EO approach.
\section{B-spline Curve and Replanning}
\label{sec:Bspline_property}
\subsection{Local Control Property and Replanning}
We use uniform B-spline as the trajectory parameterization method. The evaluation of B-spline of degree $k-1$ with uniform knot sequence $\{ t_0, t_1, t_2, \ldots, t_n \}$ of fixed time interval $\Delta t$ can be evaluated using the following equation:
\begin{equation}
\mathbf{c}(u) = \sum_{i=0}^{n} \point{v}_i N_{i,k}(u),
\end{equation}
where $u$ is the normalized parameter which can be computed according to $u = (t-t_i)/\Delta t$, for $t\in [t_i, t
_{i+1}]$. $ \point{v}_i \in \mathbb{R}^{3}$ is the control point at time $t_i$ and $N_{i,k}(t)$ is the blending function, which can be computed via the De Boor-Cox recursive formula\cite{qin2000bsplineMatrix}. Following the practice of B-spline which calls $k$ consecutive knots as a \textit{knot span}, we call the corresponding stacked control points as a \textit{span}. For instance, for the knot span $\{ t_0, t_1, t_2, \ldots, t_{k-1} \}$, the corresponding span is defined by $\mathbf{V}_0 \coloneqq \left[\point{v}_0\,\point{v}_{1}\cdots \point{v}_{k-1} \right] ^ {\intercal} \in \mathbb{R}^{k\times3}  $.

An important property of B-spline is the \textit{local control}. Namely, a single span of a B-spline curve is controlled only by $k$ control points, and any control point only affects $k$ spans, as shown in Fig.~\ref{fig:bspline_replan} (a). Mathematically, for the $i$-th span covering $k$ knots from $t_i$ to $t_{i+k-1}$, the corresponding curve and its $l$-th order derivative can be evaluated in closed-form as follows:
\begin{equation}
\label{eq:bspline_derivative}
\frac{d \mathbf{c}_i(u)}{d^{l}u} = \frac{1}{ {(\Delta t)}^l} \frac{d \mathbf{b}^{\intercal}}{d^l u} \mathbf{M}_k \mathbf{V}_i,
\end{equation}
where  $\mathbf{b} = \left[ 1 \, u \, u^2\, \cdots\, u^{k-1} \right]^{\intercal} \in \mathbb{R}^k$ denotes the basis vector,  $\mathbf{M}_k = (m_{i,j})\in \mathbb{R}^{k\times k}$ denotes the basis matrix, where $m_{i,j} = \frac{1}{(k-1)!}{k-1 \choose k-1-i}\sum_{s=j}^{k-1}(-1)^{s-j}{k \choose s-j}(k-s-1)^{k-l-i}$, and $\mathbf{V}_i = \left[\point{v}_i\,\point{v}_{i+1}\cdots \point{v}_{i+k-1} \right] ^ {\intercal} \in \mathbb{R}^{k\times3} $ stacks the $k$ control points of the $i$-th span.

\begin{figure}[t]
	\centering
	\vspace{+0.25cm}
	\includegraphics[width=0.40\textwidth]{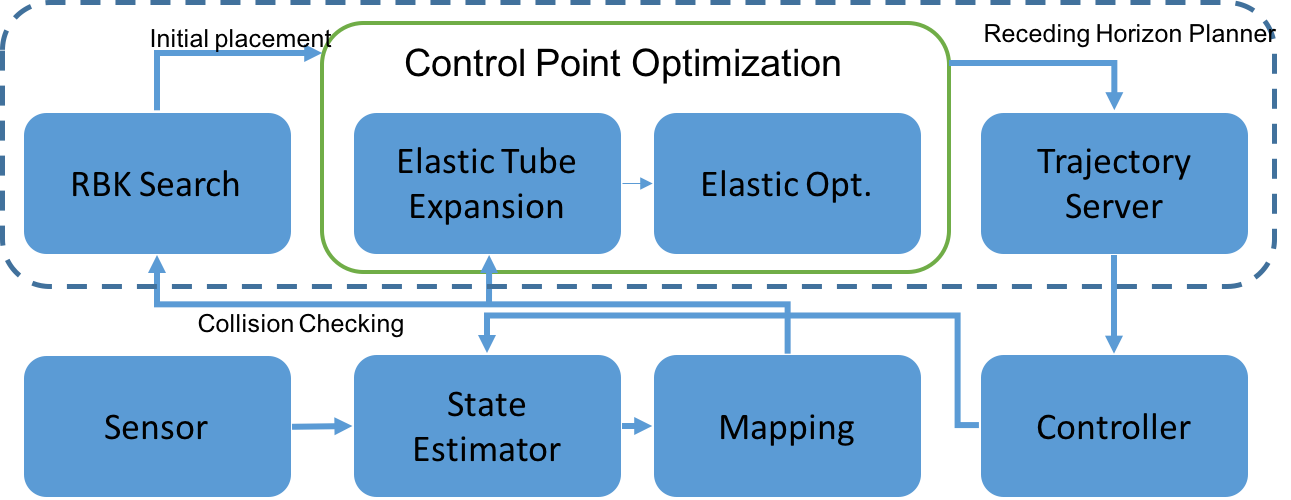}
	\caption{A diagram of our autonomous replanning system.}
	\label{fig:system overview}
	\vspace{-1.1cm}
\end{figure}

\begin{figure}[t]
	\vspace{+0.3cm}
	\begin{center}
		\subfigure[]{\includegraphics[width=0.23\textwidth]{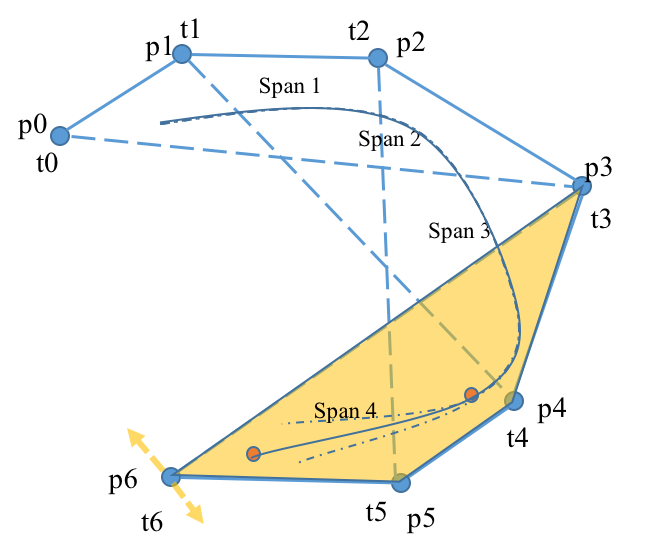}}
		\subfigure[]{\includegraphics[width=0.43\textwidth]{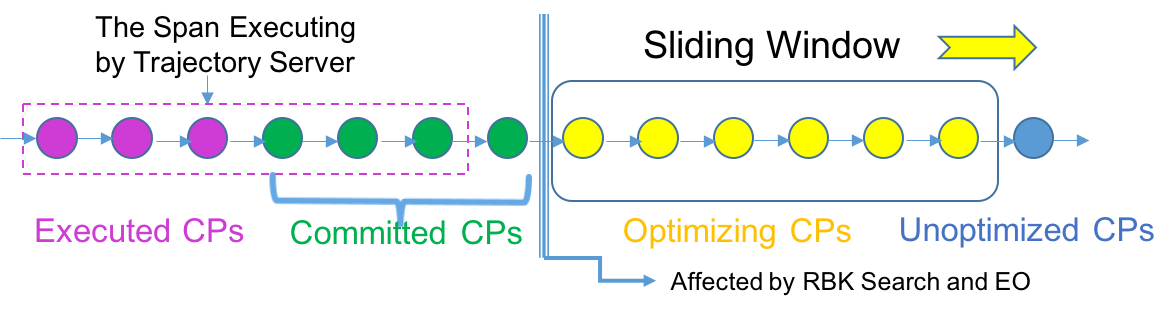}}
	\end{center}
	\vspace{-0.3cm}
	\caption{(a) shows an example of cubic B-spline. $\{p_0,\ldots,p_6\}$ denote the six control points, while $\{t_0,\ldots, t_6\}$ denote the corresponding knots. Adjusting $p_6$ will only affect the trajectory corresponding to span 4. (b) shows how the local control property can be applied to the replanning system.}
	\label{fig:bspline_replan}
	\vspace{-0.65cm}
\end{figure}

We incorporate the local control property into both a kinodynamic search process and a receding horizon planning framework. The usage of local control property in the kinodynamic search is elaborated in Sec.~\ref{sec:kinodynamic_search} and its usage in the receding horizon planning is shown in Fig.~\ref{fig:bspline_replan} (b). The control points are divided into four types, namely, executed control points, committed control points, optimizing control points (the control points inside the sliding window) and unoptimized control points.
The control point is unchangeable once committed. A stopping policy will be activated if a collision is detected for the committed trajectory. The RBK search and EO approach only affect the optimizing and unoptimized control points. As a result, the replacement caused by either the search or optimization will not cause any disturbance to the evaluation at the trajectory server.
\subsection{Convex Hull Property and Dynamical Feasibility}
Another important property of B-spline is the \textit{convex hull} property. Given Eq.~\ref{eq:bspline_derivative}, the dynamical feasibility constraints can be expressed by maximum derivative (velocity, acceleration, jerk, etc.) bounds in terms of control point placement. After applying the \textit{convex hull} property, different order derivatives of the whole B-spline curve can be completely bounded using only linear constraints, as in the following Prop.~\ref{prop:linear_derivative_bound}. The limitation of Prop.~\ref{prop:linear_derivative_bound} is its conservativeness, given that Prop.~\ref{prop:linear_derivative_bound} is a sufficient but not necessary condition.
\begin{prop}
	\label{prop:linear_derivative_bound}
	Given uniform B-spline of order $k-1$ and time step $\Delta t$, there exists a \textbf{constant} linear combination $\mathbf{S} = \mathbf{M}_k^{-1} \mathbf{C}_l \mathbf{M}_k /(\Delta t) ^{l} \in \mathbb{R}^{k \times k}$, such that  $|\mathbf{S}\mathbf{v}_i^{D}| \leq u_{l,D}^{\text{max}}\bm{1}_{k\times 1} $ is a sufficient condition for the derivative along coordinate $D$ to be thoroughly bounded for the whole span, i.e., $ \left( \frac{d c(u)}{d^l u} \right)_D \leq u_{l,D}^{\text{max}}$,$\forall u\in [0,1]$, where $\mathbf{v}_i^{D} \in \mathbb{R}^k$ is the stacked position vector of coordinate $D \in \left\lbrace X, Y, Z \right\rbrace $ in $i$-th span, and $\mathbf{C}_l \in \mathbb{R}^{k\times k}$ is a constant mapping matrix of the $l$-th derivative satisfying $ \frac{d\mathbf{b}}{d^l u} = \mathbf{C}_l \mathbf{b}$.  \footnote{ $|\cdot|$ means the element-wise absolute value of a vector.}
\end{prop}
\begin{proof}
Interested readers may refer to our report \cite{ding18replanning} for detailed proof.
\end{proof}

As described by Mellinger\cite{mellinger2011minsnap}, the control cost of quadrotor trajectory can be expressed by the integral over squared derivatives (such as fourth derivative snap), which can also be evaluated in closed form in the case of uniform B-spline. The total control cost $E_c$ of the $i$-th span can be expressed by a weighted sum of the integral over squared derivatives of different orders as follows:
\begin{equation}
\label{eq:control_cost}
E_i^c = \sum_{l=1}^{k-2} \int_{0}^{1}   w_l \left( \frac{d c_i(u)}{d^{l}u} \right)^2 du
= \sum_{l=1}^{k-2} w_l \mathbf{V}_{i}^{\intercal} \mathbf{M}_k^{\intercal} \mathbf{Q}_l \mathbf{M}_k \mathbf{V}_i
\end{equation}
where $\mathbf{Q}_l = \int_{0}^{1} \left( \frac{d \mathbf{b}}{d^l u} \right) \left( \frac{d \mathbf{b}}{d^l u} \right)^{\intercal} du /(\Delta t)^{2l-1}$ is the Hessian Matrix of the $l$-th squared derivative, which is constant for uniform B-spline, and $w_l$ is the corresponding weight. Without special mention, our replanning system uses quintic uniform B-spline ($k=6$) to ensure the continuity up to snap for quadrotor systems.

\section{B-spline Based Kinodynamic Search}
\label{sec:kinodynamic_search}
\subsection{Basics of B-spline Based Kinodynamic Search}
In this section, we introduce the basics of B-spline based kinodynamic search. A uniform distributed $M$-connect spatial grid with $N$ cells is chosen as the graph structure $\mathcal{G}$, with cell centers as vertices $\mathcal{V}$, and their connections as edges $\mathcal{E}$, i.e., $\mathcal{G} \coloneqq ( \mathcal{V},\mathcal{E})$. For instance, a 3-D grid with each cell connected to its neighbors is 26-connect. Two spans $\mathbf{V}_i$ and $\mathbf{V}_j$ are said to be neighboring spans if and only if the last $k-1$ control points of $\mathbf{V}_i$ coincide with the first $k-1$ control points of $\mathbf{V}_j$. We define the initial system state as a given span $\mathbf{V}_{\text{init}} \coloneqq \left[\point{v}_0, \point{v}_1, \ldots \point{v}_{k-1}  \right]^{\intercal}$ containing the first $k$ control points. In the same way, goal state $\mathbf{V}_{\text{goal}}$ can be defined. Note that the initial and goal state we define are slightly different from those in traditional kinodynamic planners \cite{webb2013kinodynamic}, given the fact that $\mathbf{V}_{\text{init}}$ and $\mathbf{V}_{\text{goal}}$ actually represent two short trajectories. The B-spline based kinodynamic search problem is finding a set of neighboring spans $ \mathcal{S} \coloneqq \left\lbrace \mathbf{V}_0,\ldots,\mathbf{V}_{n-k + 1}\right\rbrace$ that minimizes the following cost function,
\vspace{-0.2cm}
\begin{equation}
\begin{aligned}
& \underset{ \mathcal{S} }{\text{min}} & & J(\mathcal{S}) = \sum_{i=0}^{n-k+1} E_i^c + \lambda (n+1)\Delta t \\
\end{aligned}
\label{prob:optimal_search}
\end{equation}
\vspace{-0.05cm}
where $\lambda \geq 0$ is the weight of total time cost. The \textit{dynamical feasibility} constraints can be checked in terms of span. The \textit{collision-free} constraints are enforced by choosing control point placement in unoccupied cells. Later we will discuss how to ensure the safety of resulting trajectory based on this.

To solve the graph search problem optimally, we need a \textit{full-scale} search in terms of B-spline spans. The complexity of the full-scale search depends on the number of all possible B-spline span patterns. Typically, searching for $k-1$-order B-spline spans on an $M$-connected grid will induce an $M$-connected span graph $\mathcal{G}^{'}$ with a scaling factor $O(M^{k-1})$ of the number of vertices, if no pruning is applied. For instance, for 5-order B-spline search on a 26-connect 3-D grid, the vertex scaling factor is almost $10^7$, which is unacceptable.

\subsection{Real-time B-spline Based Kinodynamic Search}
\label{sec:rbk}
Given the high complexity of the full-scale search, we present the RBK search algorithm, which transforms the low-complexity position-only search into an efficient kinodynamic search, by exploring three properties of B-spline, namely, \textit{local control} property, \textit{convex hull} property and \textit{closed-form evaluation} of control cost. The three properties exactly correspond to the three core operations in the  search process, namely, span retrieving and expansion, dynamical feasibility checking and cost evaluation. The RBK algorithm shares a similar structure and complexity to position-only search algorithms such as A* and Dijkstra, but obtains the ability of evaluating the control cost and ensuring the dynamical feasibility. The result of the RBK search is dynamically feasible time-parameterized trajectory instead of unparameterized path, thus being different from position-only search. Denote open set and closed set as $\mathcal{O}$ and $\mathcal{I}$, and current grid node and neighboring grid node as $p_\text{cur}$ and $p_\text{nbr}$.

\begin{figure}[t]
	\removelatexerror
	\begin{algorithm}[H]
		\label{algo:RBK}
		\caption{ $\func{RBK} (\mathbf{V}_{\text{init}},\mathbf{V}_{\text{goal}}, \set{G}   )$ }
		Initializes: $\mathcal{O} = \mathcal{I} = \emptyset$\;
		$\func{Add}(\mathcal{O} , p_{\text{new}}, \func{Cost}(\mathbf{V}_{\text{init}}))$ \;
		\While(){ $\func{size} (\mathcal{O} ) != 0$ }
		{
			$ p_{\text{cur}} \leftarrow \func{PopMin}(\mathcal{O} )$,	$ \func{Add}(\mathcal{I}, p_{\text{cur}})$ \;
			$(\mathbf{V}_\text{cur},c_{\text{cur}}) \leftarrow \underline{ \func{RetrieveSpan} }(p_{\text{cur}}) $ \;
			\If{$\func{NearEnd}(\mathbf{V}_\text{cur}, \mathbf{V}_\text{goal})$}
			{
				\textbf{return success} \;
			}
			\For(){$ \mathbf{V}_{\text{nbr}} \in \func{Neighbor}(\mathbf{V}_\text{cur}) \cap p_{\text{nbr}} \notin \mathcal{I} $}
			{
				\If{ $\underline{\func{CheckDynamics} }(\mathbf{V}_{\text{nbr}})$}
				{
					$ c^{\prime} = \underline{ \func{Cost} }(\mathbf{V_{\text{nbr}}}) + c_{\text{cur}}$ \;
					\eIf{$ p_{\text{nbr}} \in \mathcal{O}  $}
					{
						\If{$c^{\prime} < \func{RetrieveCost}(p_{\text{nbr}})$}
						{
							$ \func{Update} (p_{\text{current}}, p_{\text{nbr}}) $ \;
						}

					}{
					$\scriptstyle \func{Add}(\mathcal{O} , p_{\text{nbr}},c^{\prime} + \func{HeuristicCost}(\mathbf{V}_{\text{nbr}}) )$ \;
				}
			}
		}

	}
\end{algorithm}
\vspace{-1.3cm}
\end{figure}

\underline{\textit{Span retrieving and expansion}}: the position-only search algorithms such as A* and Dijkstra implicitly maintain a tree structure via the predecessor data structure. Instead of retrieving one predecessor as in the position-only search, the RBK search retrieves $k-1$ predecessors. By stacking these $k-1$ predecessors with current control point, the current span is formed. The retrieving process is implemented in function $\func{RetrieveSpan}(p_{\text{cur}})$ in Alg.~\ref{algo:RBK} and is shown in Fig.~\ref{fig:rbk_flow} (a). Thanks to the local control property, the expansion can be done locally just like position-only search, since there is no need to re-evaluate of the whole trajectory when expanding to new control points. Therefore, using local control property, the tree structure in A* and Dijkstra can be transformed into a tree of spans, enabling B-spline based search.

\underline{\textit{Dynamical feasibility checking}}: using the \textit{convex hull} property, we derive a sufficient condition for dynamical feasibility as introduced in Sec.~\ref{sec:Bspline_property}, which is actually linear. By using this condition in function $\func{CheckDynamics}(\mathbf{V}_{\text{nbr}})$, dynamical feasibility can be \textit{guaranteed} without computation overhead, which speeds up the kinodynamic search process.

\underline{\textit{Cost evaluation}}: Traditional kinodynamic planners \cite{webb2013kinodynamic} rely on solving the two-step boundary value problem (BVP) to determine the cost of connecting two states, which is computationally expensive. By using B-spline parameterization, the cost of connecting to new expanded span can be evaluated using closed-form solution according to Eq.~\ref{eq:control_cost}, which is implemented in function $\func{Cost}(\mathbf{V_{\text{nbr}}})$ in Alg.~\ref{algo:RBK}. The limitation of our method is that the expansion of spans are restricted by the resolution and connection of the grid. For example, for 26-connect 3-D grid, the connection to the next control point is restricted to the 26 connections, resulting in limited representations of B-spline trajectories.

\begin{figure}[t]
	\vspace{+0.0cm}
	\centering
	\subfigure[]{\includegraphics[trim={0cm 0cm 0cm 0cm},clip,width=0.21\textwidth]{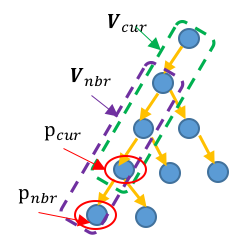}}
	\subfigure[]{\includegraphics[trim={0cm 0cm 0cm 0cm},clip,width=0.23\textwidth]{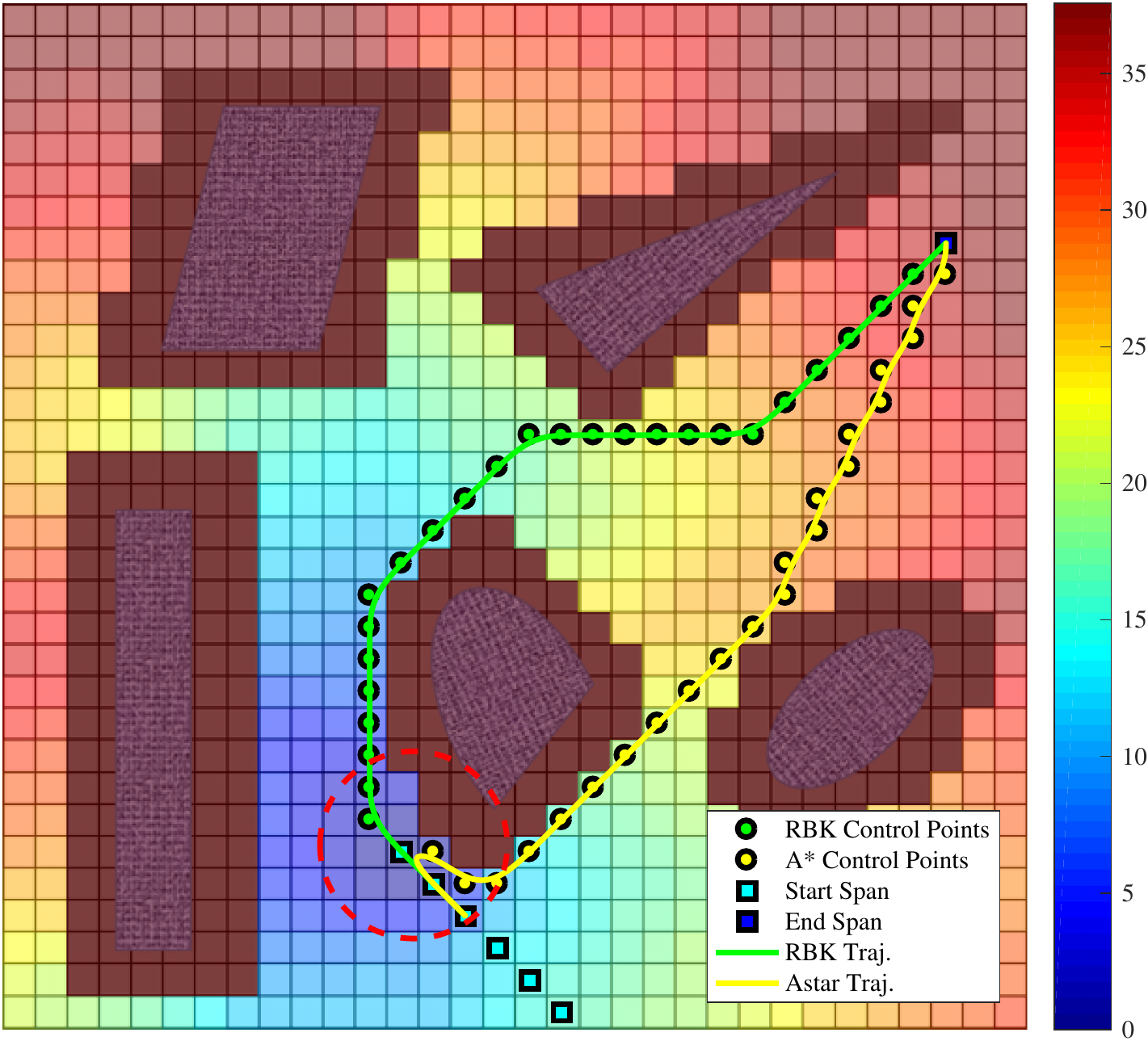}}
	\vspace{-0.2cm}
	\caption{Illustration of the RBK search. (a) shows the span expansion process (Sec.\ref{sec:rbk}), and (b) shows the result of the RBK search and the comparison with the position-only A* search.}
	\label{fig:rbk_flow}
	\vspace{-0.8cm}
\end{figure}

The trajectory of the RBK search can be \textit{guaranteed} to be collision free if moderate obstacle inflation is done. The analysis of how much inflation is needed is shown in our report \cite{ding18replanning}, which means that the RBK search can be used as a standalone component without any post-processing, when computation is limited. The termination condition is implemented in $\func{NearEnd}(\mathbf{V}_\text{cur}, \mathbf{V}_\text{goal})$ in Alg.~\ref{algo:RBK}. In practice, we terminate the search when the first control point of end span is reached and append the end span to it, since we observe that the part of dynamical profile can be automatically smoothed by the back-end. To speed up the search process, we use $\func{HeuristicCost}(\mathbf{V}_{\text{nbr}})$ to focus on searching in the most promising direction. In practice, one admissible heuristic can be $\lambda \cdot  \func{norm}(p_{\text{nbr}} - p_{\text{goal}})/v_\text{max}$, where $p_{\text{goal}}$ is the first control point of end span. The design of more informative heuristic function is left as a future work.

As shown in Fig.~\ref{fig:rbk_flow} (b), based on the start span, the RBK search explores the low-cost region and generates dynamically feasible time-parameterized trajectory, while A* finds a shortest but unparameterized path, which does not contain any kinodynamic information. If we directly use the A* shortest path as B-spline control points placement, the resulting trajectory has large acceleration at the beginning, which is not good for replanning. As verified in Sec.~\ref{sec:rbk_performance}, the RBK search generally finds better trajectories in terms of the control cost and time cost, compared to the position-only A*. Moreover, the dynamical feasibility is guaranteed for the RBK search. As verified in Sec.~\ref{sec:analysis} and Sec.~\ref{sec:experimental}, Alg.~\ref{algo:RBK} only needs around 20 milliseconds to reach a good solution.


\section{Elastic Optimization}
\label{sec:elastic_optimization}
To compensate for the potential sub-optimality of the RBK search and the limitation induced by the discretization of the grid, we present the post-optimization process, i.e., the EO approach, to fully utilize free space and further reduce the trajectory cost. Since the initial trajectory provided by the RBK search is already dynamically feasible and time-parameterized, the post-optimization is equipped with at least one feasible solution, which makes the optimization process robust. The EO approach can be divided into two components, namely, an elastic tube expansion algorithm and an elastic optimization formulation. The elastic tube expansion algorithm generates a series of connected local maximum volume ``bubbles'', i.e., an elastic tube. The elastic optimization is to stretch the trajectory (so-called ``elastic band'') inside the tube, so that the objective is minimized.

\subsection{Elastic Tube Expansion}
\label{sec:elastic_expansion}
We denote the set of control point placement provide by Alg.~\ref{algo:RBK} as $\mathcal{P} \coloneqq \left\lbrace \point{p}_0, \point{p}_1, \ldots, \point{p}_m   \right\rbrace $ as the initial placement, excluding the first $k-1$ control points of the start span, and the last $k-1$ control points of the end span. As shown in Alg.~\ref{algo:elastic_tube}, the elastic tube expansion algorithm can be divided into two steps: first, we construct the initial tube, by doing radius search for the initial placement $\set{P}$ and get the nearest obstacle position $\mathbf{n}_i$. Second, we push the center of the bubbles in the direction $\overrightarrow{f}$ (away from the nearest obstacle), while satisfying the criterion that the new bubble contains the original bubble, as required by condition $\func{abs} (r_i^{\prime} - d - r_i )\leq d_{\text{thres}}$, as shown in Fig.~\ref{fig:eo_pipeline} (a). The inflation process is implemented in a binary search manner to reduce calling $ \func{NNSearch}$.  Alg.~\ref{algo:elastic_tube} will finally find a series of local maximum volume bubble centers $\set{Q} \coloneqq \left\lbrace \point{q}_0, \point{q}_1,\ldots,\point{q}_m \right\rbrace  $ based on the initial tube $\set{P}$. As for the parameter settings,  $d_\text{infl}^{\max}$ and $ d_\text{infl}^{\min}$ are the maximum and minimum inflation distance respectively; $d_{\text{thres}}$ is the threshold for checking whether the new bubble contains the original one, and should be set to a small value, e.g., less than the map resolution; and $d_\text{infl}^{\text{tol}}$  is the binary search end condition, which can be set to the resolution of the map. The function $\func{NNSearch}$ is the nearest neighborhood search which can be done efficiently if a KD-tree is maintained. The actual run time of Alg.~\ref{algo:elastic_tube} is around 2 millisecond, as verified in Section.~\ref{sec:analysis}.

\begin{figure}[t]
	\removelatexerror
	\setlength{\textfloatsep}{0pt}
	\begin{algorithm}[H]
		\label{algo:elastic_tube}
		\caption{ $(\set{Q}, \set{R}^{\prime}) \leftarrow \func{ElasticTube} (\set{P}, \set{C}^{\text{ELAS}} )$ }
		Initializes: $d_\text{infl}^{\min}, d_\text{infl}^{\max}, d_{\text{thres}}.\, \set{R} = \set{R^{\prime}} = \set{Q} = \emptyset$ \;
		\For{ $\point{p}_i \in \set{P}$ }
		{
			$(\point{n}_i,r_i)\leftarrow \func{NNSearch}(\point{p}_i,\set{C}^{\text{ELAS}})$\;
			$\overrightarrow{f} = (\point{p}_i - \point{n}_i)/\func{norm}(\point{p}_i - \point{n}_i) $ \;
			\While{ $d_\text{infl}^{\max} - d_\text{infl}^{\min} > d_\text{infl}^{\text{tol}}$ }
			{
				$d \leftarrow \left( d_\text{infl}^{\max} + d_\text{infl}^{\min} \right) /2 $, $\point{p}_{i,\text{infl}}  \leftarrow \point{p}_i + d \cdot \overrightarrow{f} $   \;

				$(\point{n}_i^{\prime},r_i^{\prime})\leftarrow \func{NNSearch}(\point{p}_{i,\text{infl}},\set{C}^{\text{ELAS}})$\;
				\eIf{ $\func{abs} (r_i^{\prime} - d - r_i ) > d_{\text{thres}} $}
				{
					$d_\text{infl}^{\max}  \leftarrow d $ \;
				}{
				$d_\text{infl}^{\min}  \leftarrow d $ \;
			}
		}
		$\mathcal{Q} \leftarrow \mathcal{Q} \cup \point{p}_{i,\text{infl}}$, $\set{R}^{\prime} \leftarrow \set{R}^{\prime} \cup r_i^{\prime}$ \;

	}
\end{algorithm}
\vspace{-1.2cm}
\end{figure}

\subsection{Elastic Optimization Formulation}
\label{sec:elastic_opt_formulation}
We are inspired by the elastic smoothing method for car-like robots in \cite{zhu2015ces}. The key insight of \cite{zhu2015ces} is the convex representation of geometric smoothness and enforcement of speed feasibility. However, the formulation in \cite{zhu2015ces} is completely based on car model. In contrast to \cite{zhu2015ces} which worked on a tube limited by the initial path, we use Alg.~\ref{algo:elastic_tube} to generate a local maximum volume elastic tube to enhance the optimization performance. Furthermore, we represents the control cost and dynamical feasibility constraints in closed-form using the high-order B-spline parameterization for controlling quadrotors.

We denote $\left\lbrace \point{x}_0, \point{x}_1, \ldots, \point{x}_{m} \right\rbrace $ as the control point placement (optimization variable). We denote the stacked spans as $\left\lbrace  \mathbf{V}_i | i\in \left\lbrace -k+1,-k+2,\ldots m \right\rbrace \right\rbrace  $ which also contain the hybrid spans which consist of optimization variable and fixed start span or end span control points. The optimization problem can be expressed as follows:
\begin{equation}
\begin{aligned}
& \text{min} & & \sum_{i=-k+1}^{m} E_i^c\\
& \text{s.t} 						   & & \point{x}_0 = \point{p}_0, \point{x}_m = \point{p}_m\\
&								       & & \left\Vert \point{x}_i - \point{q}_m \right\Vert_2 \leq r_i^{\prime}, & \forall i \in \left\lbrace 2,\ldots, m-1 \right\rbrace\\
&									   & &|\mathbf{S}\mathbf{v}_i^{D}| \leq u_{l,D}^{\text{max}}\bm{1}_{k\times 1}, &\forall i, D, l\in \left\lbrace 1,\ldots k-2 \right\rbrace  \label{prob:ces_opt}
\end{aligned}
\end{equation}
where the first constraint is the position constraint of the first and last control point to maintain the connectivity. The second constraint is to restrict the control point inside the safety space, which is quadratic. The third constraint is to ensure the dynamical feasibility using the sufficient condition in Prop.~\ref{prop:linear_derivative_bound}, which is linear. As a whole, the formulation is a QCQP. As verified in \ref{sec:analysis}, the EO approach can be completed in around 30 milliseconds.

\begin{figure}[t]
	\begin{center}
		\subfigure[]{\includegraphics[trim={0cm 0cm 0cm 0cm},clip,width=0.233\textwidth]{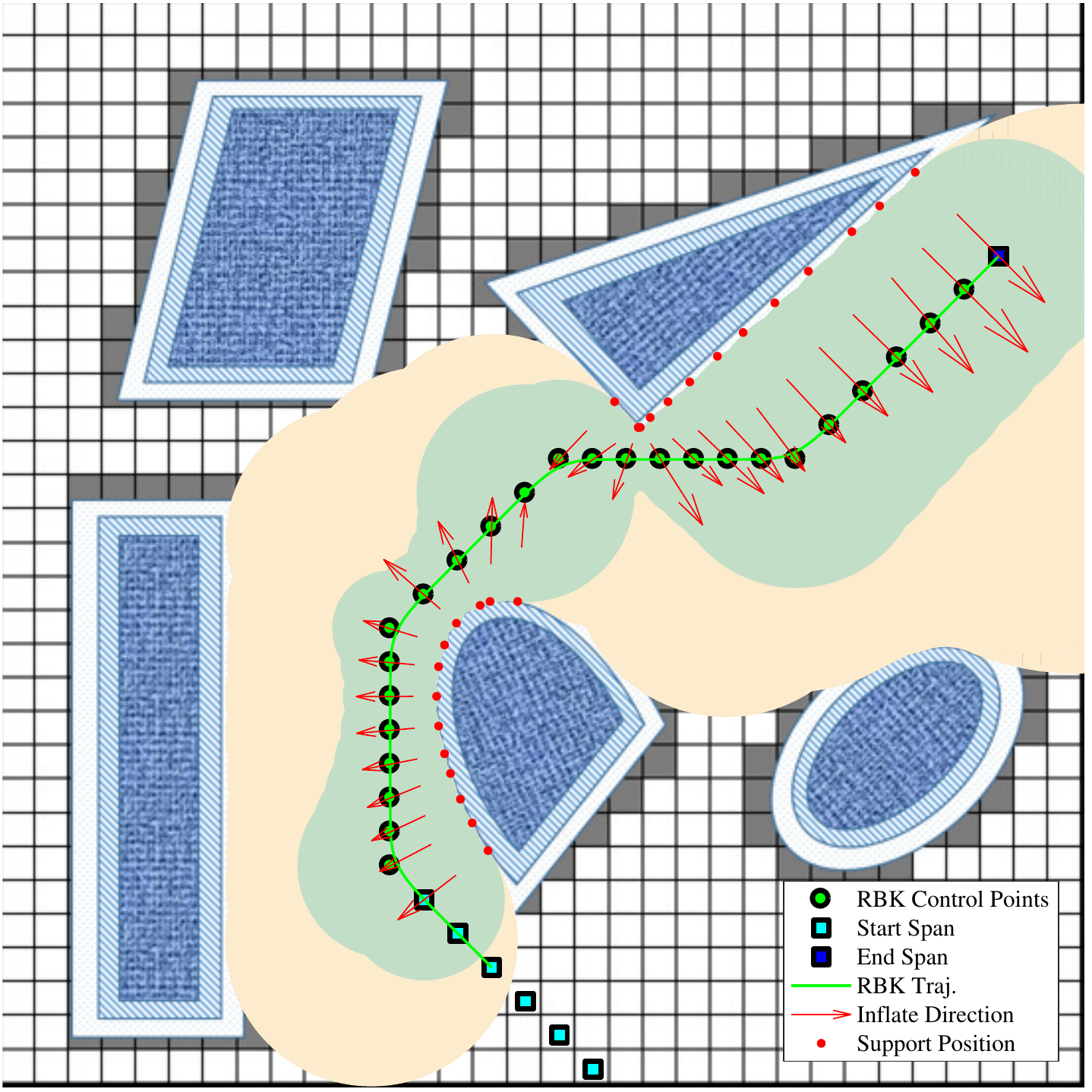}}
		\subfigure[]{\includegraphics[trim={0cm 0cm 0cm 0cm},clip,width=0.233\textwidth]{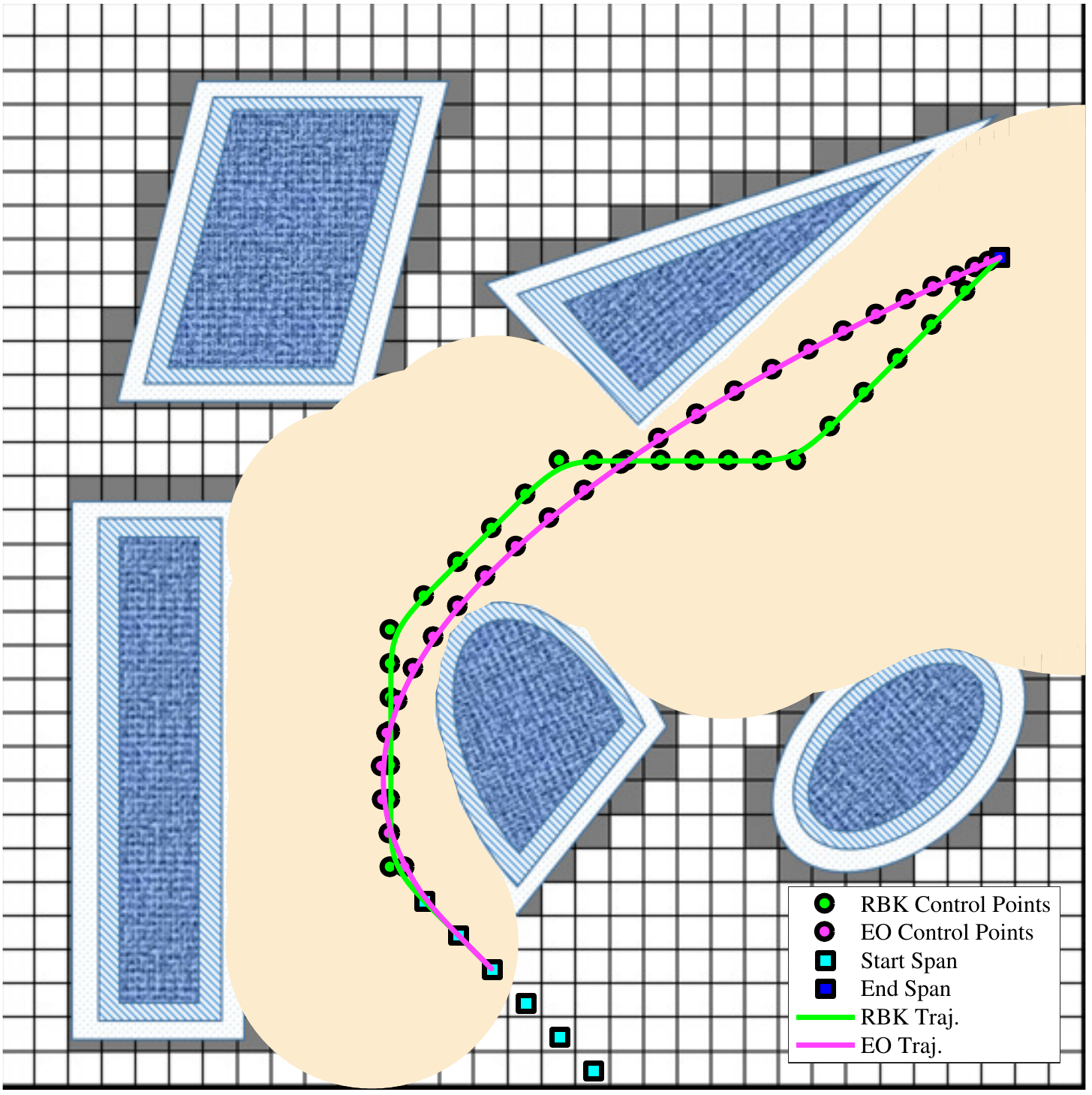}}
	\end{center}
	\vspace{-0.35cm}
	\caption{Illustration of elastic optimization approach. (a) shows the elastic tube expansion process (Sec.~\ref{sec:elastic_expansion}). (b) shows the elastic optimization process (Sec.~\ref{sec:elastic_opt_formulation}).}
	\label{fig:eo_pipeline}
	\vspace{-1.1cm}
\end{figure}

\subsection{Enforcing Safety Guarantee}
\label{sec:safety_guarantee}
One issue of the elastic optimization formulation is that restricting control points inside the bubbles is not a sufficient condition for safety, since 1) bubbles are placed with finite density and 2) B-spline does not exactly pass the control points. The issue can be resolved in two steps: first we show that piece-wise linear segments connecting the control points can be completely contained in collision-free space using a two-level obstacle inflation strategy; and second, the B-spline trajectory can be pulled arbitrarily closed to the piece-wise linear segments using an iterative process of adding control points.
The two-level obstacle inflation strategy is as follows: we model the robot as a ball with radius $\delta_r$, and we search in configuration space $\mathcal{C}^{\text{RBK}}$ with a larger minimum clearance $c_{\min}^{\text{RBK}}$, while we generate the elastic tube and optimize the control points in the configuration space $\mathcal{C}^{\text{ELAS}}$ with a smaller minimum clearance $c_{\min}^{\text{ELAS}}$ as shown in Fig.~\ref{fig:eo_pipeline}.
Based on the given grid size of RBK search, the maximum distance between control points $d_\text{max}$ can be computed. By the two-level inflation, we could have $2(c_{\min}^{\text{RBK}} - c_{\min}^{\text{ELAS}}) > d_\text{max}$, which is used to ensure the connectivity of the initial tube. Based on this, a sufficient condition for the piece-wise linear segments to be completely contained inside the elastic tube is that $c^{\text{ELAS}}_\text{min} - \delta_r> \frac{\sqrt{2} - 1}{2} d_\text{max}$ \footnote{Interested readers may also refer to our report \cite{ding18replanning} for detailed proof.}. Then, if the collision is caused by the deviation from the piece-wise linear segments, control points are added till the collision is resolved. Thanks to the local control property, this can be done with high efficiency without re-evaluation of the whole trajectory. In the case of collision, we can always \textit{guarantee} the success of iterative process by realizing that if the same $k-1$ control points are added to one control point, the robot will exactly pass through that control point. The linear segment can be closely followed by adding control points on both sides. \footnote{The strategy is to guarantee collision-free for the post-processing.}
\section{Analysis}
\label{sec:analysis}
We begin with an individual test for the RBK search, then we test the replanning system. A set of default parameters is used for all the simulation: 1) Maximum velocity and acceleration are set to  $3.5 m/s$ and $5.2 m/s^2$ respectively. Accordingly, $u^{\text{max}}_{1,D} = 2.0 m/s$  (velocity bound for each axis) and  $u^{\text{max}}_{2,D} = 3.0 m/s^2$. 2)The B-spline order is set to $5$, and the time step $\Delta t$ is set to $0.15s$. 3) The control cost in Eq.\ref{eq:control_cost} is set to minimizing snap (only $w_4$ being non-zero). The weight of total time cost $\lambda$ in Eq.~\ref{prob:optimal_search} is set based on the criterion that the time cost and the control cost are at the same order for each span. 4) For the replanning system, the RBK search is implemented on a $60 \times 60 \times 20$ uniform grid, with an actual scale of $10m \times 10m \times 3m$. The sliding window size is set to $12$, and the local sensing range is set to $4m$.

\subsection{Kinodynamic Search Performance}
\label{sec:rbk_performance}
To verify the performance of the RBK search, we compare it with the full-scale B-spline based kinodynamic search and position-only A* path search. Since the full-scale B-spline based kinodynamic search is of high complexity, the comparisons are done on a relatively small grid of $14 \times 14$.  We can not directly compare our method with the position-only A* search according to the cost function defined in Eq.~\ref{prob:optimal_search}, since the path returned by the position-only search does not contain any kinodynamic information. Hence, we assume that the position-only search results are directly parameterized using B-spline, to illustrate the performance if a position-only search is used as the system front-end.

\begin{figure}[htb]
	\vspace{+0.5cm}
	\begin{center}
		\subfigure[]{\includegraphics[trim={0cm 0cm 0cm 0cm},clip,width=0.41\columnwidth]{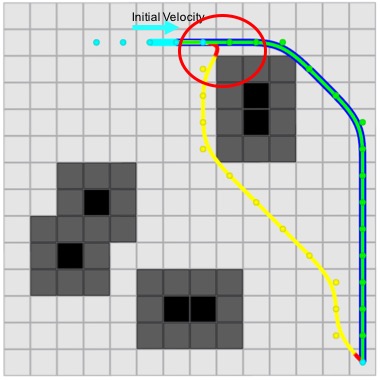}}
		\subfigure[]{\includegraphics[trim={0cm  0cm
				0cm 0cm},clip,width=0.567\columnwidth]{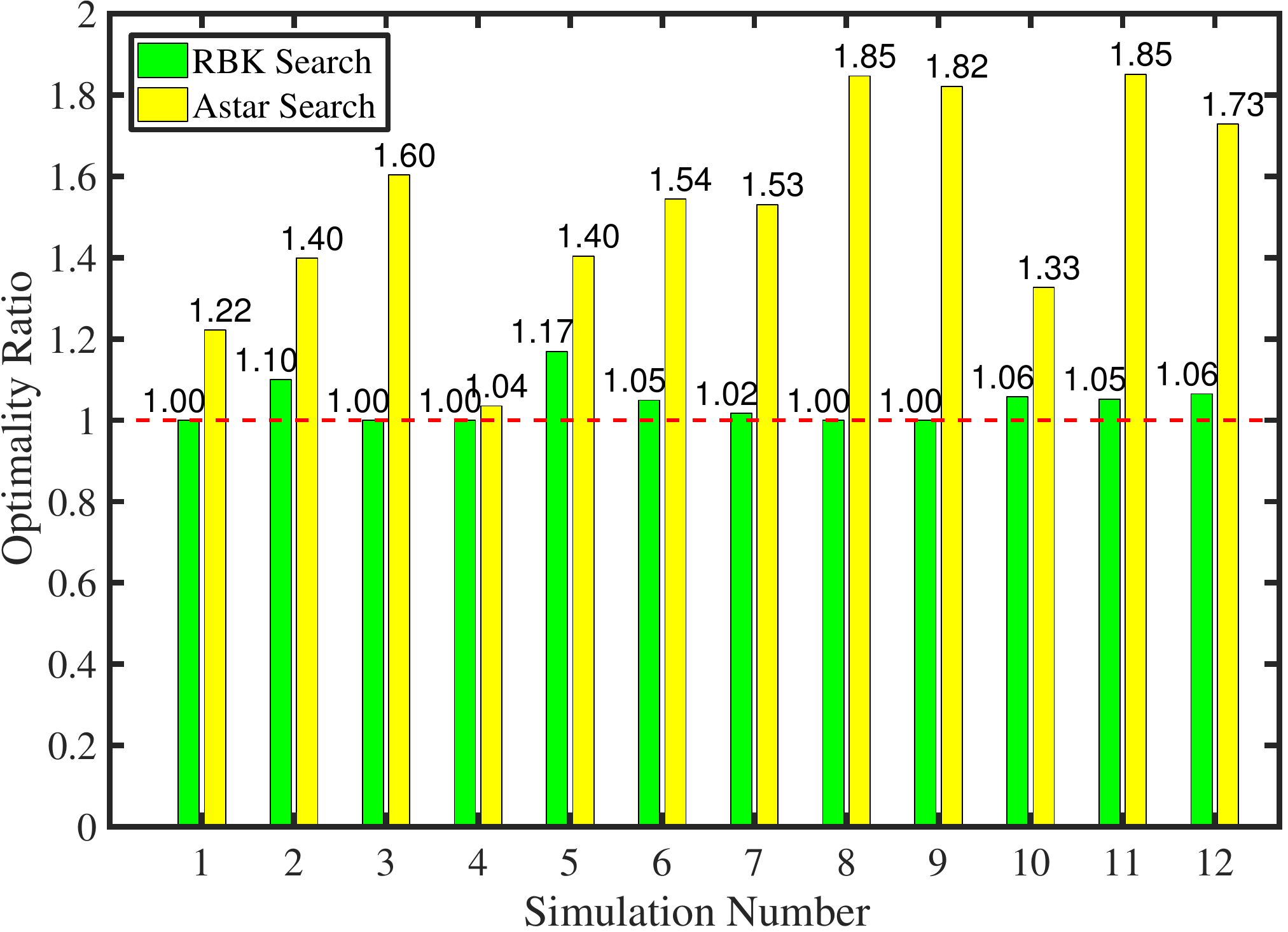}}
	\end{center}
	\vspace{-0.3cm}
	\caption{Illustration of the RBK search performance and comparisons. In (a), the start span (in cyan) has a non-zero initial velocity pointing to the right. The end span (in cyan) is at the same grid representing a static state. The path of the position-only A* search (in yellow) and dynamically infeasible parts (in red) are shown. The results of the full-scale B-spline based search and the RBK search are marked in blue and green respectively.}
	\label{fig:kinob_test}
	\vspace{-1.15cm}
\end{figure}

As shown in Fig.~\ref{fig:kinob_test} (a), if the position-only A* shortest path is used as the control point placement, there are dynamically infeasible parts as marked in red. Our RBK search and the full-scale B-spline based search both \textit{guarantee} dynamical feasibility, and in this case our RBK method actually finds the optimal solution. We further conduct Monte-Carlo testing by randomly choosing the start state (random location and random pattern of the span) and the obstacle location. As shown in Fig.~\ref{fig:kinob_test} (b), we demonstrate the optimality ratio statistics (i.e., total cost divided by the optimal cost given by the full-scale search). The RBK search finds a lower cost solution than position-only A* search for all 12 rounds and reaches the optimal solution for five rounds. As for the time efficiency, the average computing time for the position-only A*, the RBK and full-scale search are $0.0003s$, $0.0035s$ and $7.29s$, respectively. Therefore, the RBK search is about three orders of magnitude faster than the full-scale search even on the small grid. Note that more efficient heuristics can be developed to improve the RBK performance and characterize the optimality gap, and this will be left as a future work.

\subsection{Run Time Analysis of Replanning System}
\label{sec:run_time_analysis}
In this section, we test our replanning system on two different maps, including the random map and Perlin noise map\footnote{https://github.com/HKUST-Aerial-Robotics/mockamap}. The Perlin noise map is a complex 3-D world, as shown in Fig.~\ref{fig:runtime_sim}(b). We evaluate the run time efficiency of our replanning system and list the statistics of all components as shown in Table.~\ref{tab:runtime_analysis}. The overall trajectory illustrating the whole round trip is shown in Fig.~\ref{fig:runtime_sim}.
\begin{table}[h]
	\caption{Run Time Analysis on Different Maps}
	\label{tab:runtime_analysis}
	\resizebox{\columnwidth}{!}{
		\begin{tabular}{ |c|c|c|c|c|c|c|c|c|}
			\hline
			\textbf{Maps}    &
			\makecell{$\#$ \\ \textbf{Replans} }&
			\textbf{Time(s)} &
			\textbf{\makecell{RBK\\Search}} &
			\makecell{$\#$ \\ \textbf{Opt.}}&
			\textbf{Time(s)} &
			\textbf{\makecell{Tube\\Expan.}} &
			\textbf{\makecell{Traj.\\Opt.}} &
			\textbf{\makecell{Total\\ Opt. } } \\
			\hline
			\makecell{Random Map\\ ( 0.25 pillars/$m^2$ ) }  &  76  & \makecell{Avg\\Max\\Std}  & \makecell{\textbf{0.017}\\0.049\\0.010} & 993 &\makecell{Avg\\Max\\Std}  & \makecell{\textbf{0.002}\\0.009\\0.001} & \makecell{\textbf{0.021}\\0.043\\0.010}  & \makecell{\textbf{0.023}\\0.044\\0.010}\\
			\hline
			Perlin Map  &  19  & \makecell{Avg\\Max\\Std}  & \makecell{\textbf{0.014}\\0.026\\0.006} & 1044 &\makecell{Avg\\Max\\Std}  & \makecell{\textbf{0.002}\\0.008\\0.001} & \makecell{\textbf{0.028}\\0.058\\0.010}  & \makecell{\textbf{0.030}\\0.061\\0.011}\\
			\hline
		\end{tabular}
	}
	\vspace{-0.2cm}
\end{table}

As we can see from Tab.~\ref{tab:runtime_analysis}, on the random map, the RBK method consumes an average computing time of $0.017s$ with a standard deviation of $0.01s$. The elastic tube expansion method can be done in $0.002s$, which is much faster than traditional free space segmentation methods such as IRIS \cite{deits2015iris}, which may take up to $0.1s$ to find a solution. The elastic optimization can be done in $0.021s$. On the Perlin noise map, which contains unstructured 3D obstacles, our method has a similar performance as on the random map, showing that our method works well in complex 3-D environments.
\begin{figure}[t]
	\begin{center}
		\subfigure[Replan on the random map]{\includegraphics[trim={13.5cm 2cm 15cm 0cm},clip,width=0.20\textwidth]{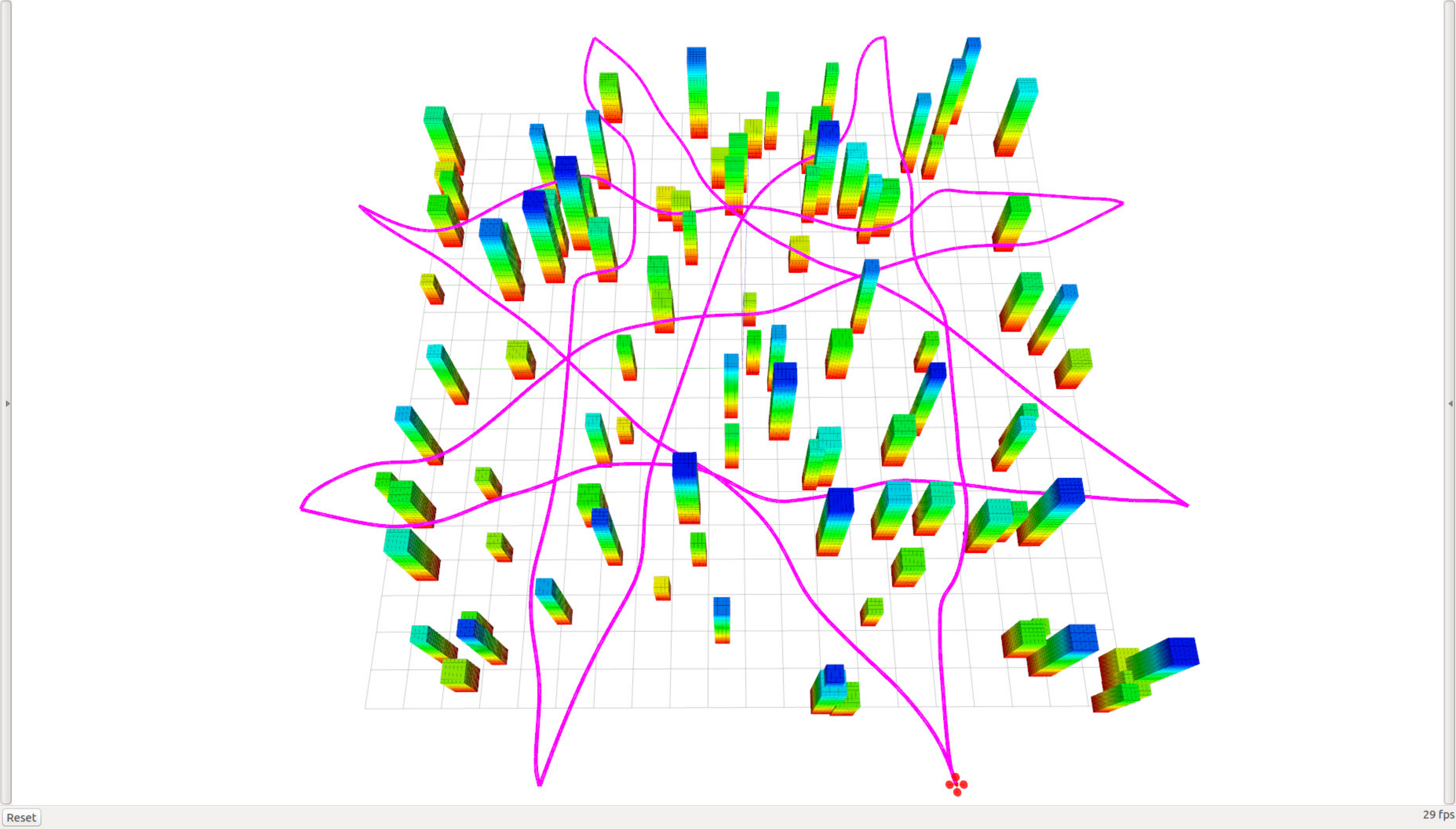}}
		\subfigure[Replan on the perlin noise map]{\includegraphics[trim={10cm 1.5cm 14cm 1cm},clip,width=0.22\textwidth]{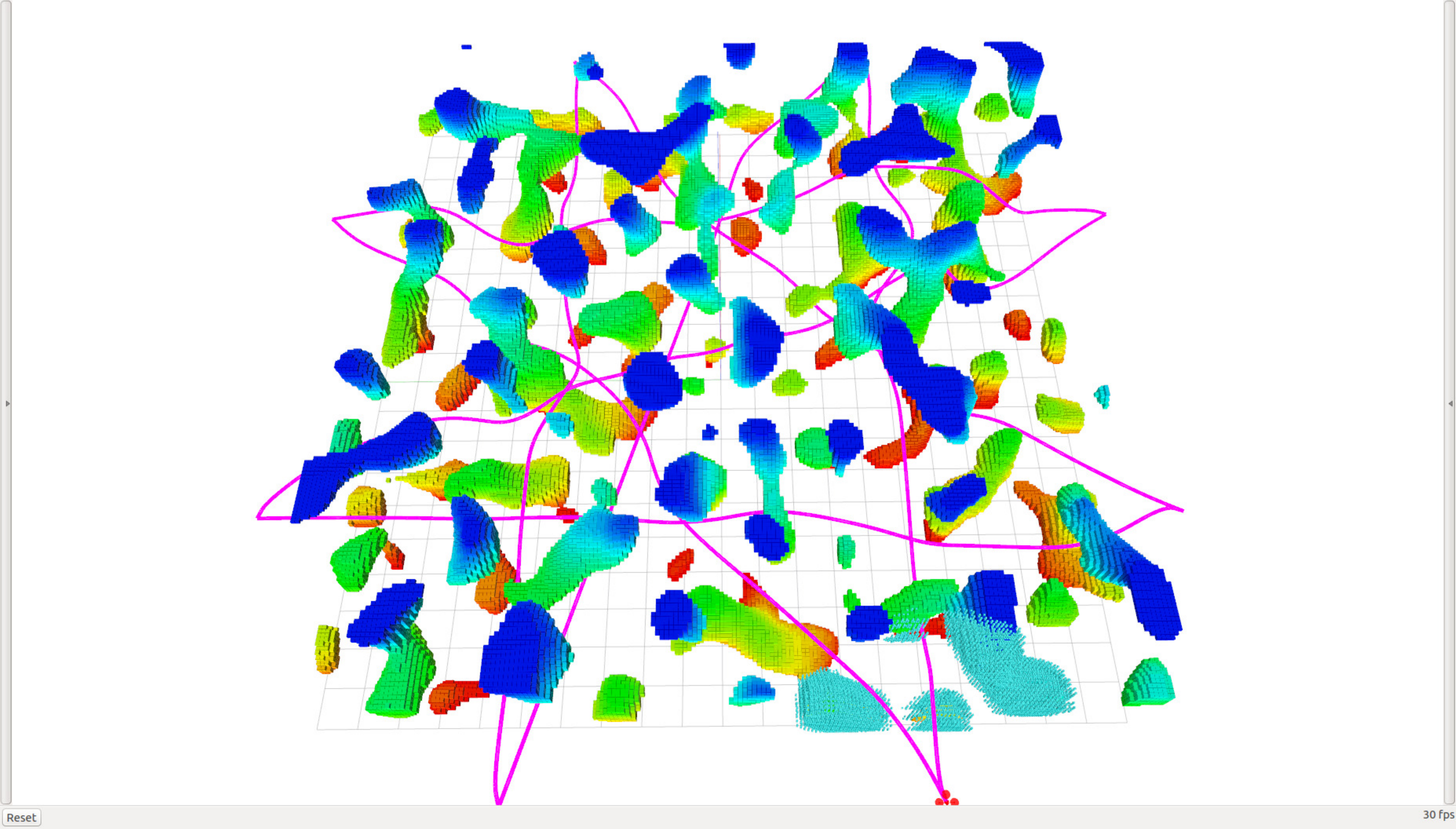}}
	\end{center}
	\vspace{-0.3cm}
	\caption{Illustration of our replanning system on different maps.}
	\label{fig:runtime_sim}
	\vspace{-1.0cm}
\end{figure}

\subsection{System Comparison With State-of-the-Art Methods}
In this section, we conduct a system comparison with two state-of-the-art methods, including the Continuous Trajectory (CT) \cite{oleynikova2016ct} method and Searched-based Motion Primitive (SMP) method by Liu \emph{et al.} \cite{Liu2017smp}. The SMP method \cite{Liu2017smp} constructs motion primitives online based on linear quadratic minimum time control, and conducts heuristic-guided search to generate dynamically feasible trajectories. We set up a challenging obstacle-cluttered 3-D complex simulation environment containing walls and 3-D steps, as shown in Fig.~\ref{fig:replan_tuple}. The replanning strategy is choosing a local target on a given straight-line guiding path. For CT method, the initial guiding path is parameterized using polynomial, following their practice. The CT method tries to use gradient-based trajectory optimization to push a short local trajectory out of the Euclidean Signed Distance Field (ESDF) formed by the obstacles. For SMP method, we use receding horizon planning in their paper \cite{Liu2017smp}, namely, for every time slot, we plan for the next time slot while executing current committed trajectory. The replanning strategy for our method is the passive mode, as introduced in Sec.\ref{sec:overview}. We evaluate the replanning system from the trajectory statistics and time efficiency perspectives (as shown in Tab.~\ref{tab:simulation_replan}, CT is excluded in the table due to its low success rate).

As shown in Fig.~\ref{fig:replan_tuple} (a), the CT method uses the ESDF to push the initial trajectory out of the obstacles, but suffers from a low success rate due to local minimum issue. The SMP method is not real-time when the full dynamic range is applied, since a large dynamic range requires a higher level of discretization of the control input, which will result in a dramatically growing state lattice. As shown in Tab.~\ref{tab:simulation_replan}, the SMP under a full dynamic bound may need up to $0.621s$ (with a $0.086s$ standard deviation), which is unacceptable for real-time replanning. If we limit the dynamic bounds to be conservative, the SMP method (so-called SMP Conservative) is efficient but scarifies the maneuverability. Note that only acceleration-controlled SMP can give a real-time performance ($10$Hz) in 3-D environments \cite{Liu2017smp}. To obtain a smooth enough control input, SMP uses the unconstrained QP formulation in \cite{richter2016polyunqp} to reparameterize the trajectory, but this post-processing has no safety or dynamical feasibility guarantee.

\begin{figure*}[t]
	\begin{center}
		\subfigure[CT Method] {\includegraphics[trim={0.2cm 0.4cm 0.2cm 0cm},clip, width=0.245\textwidth]{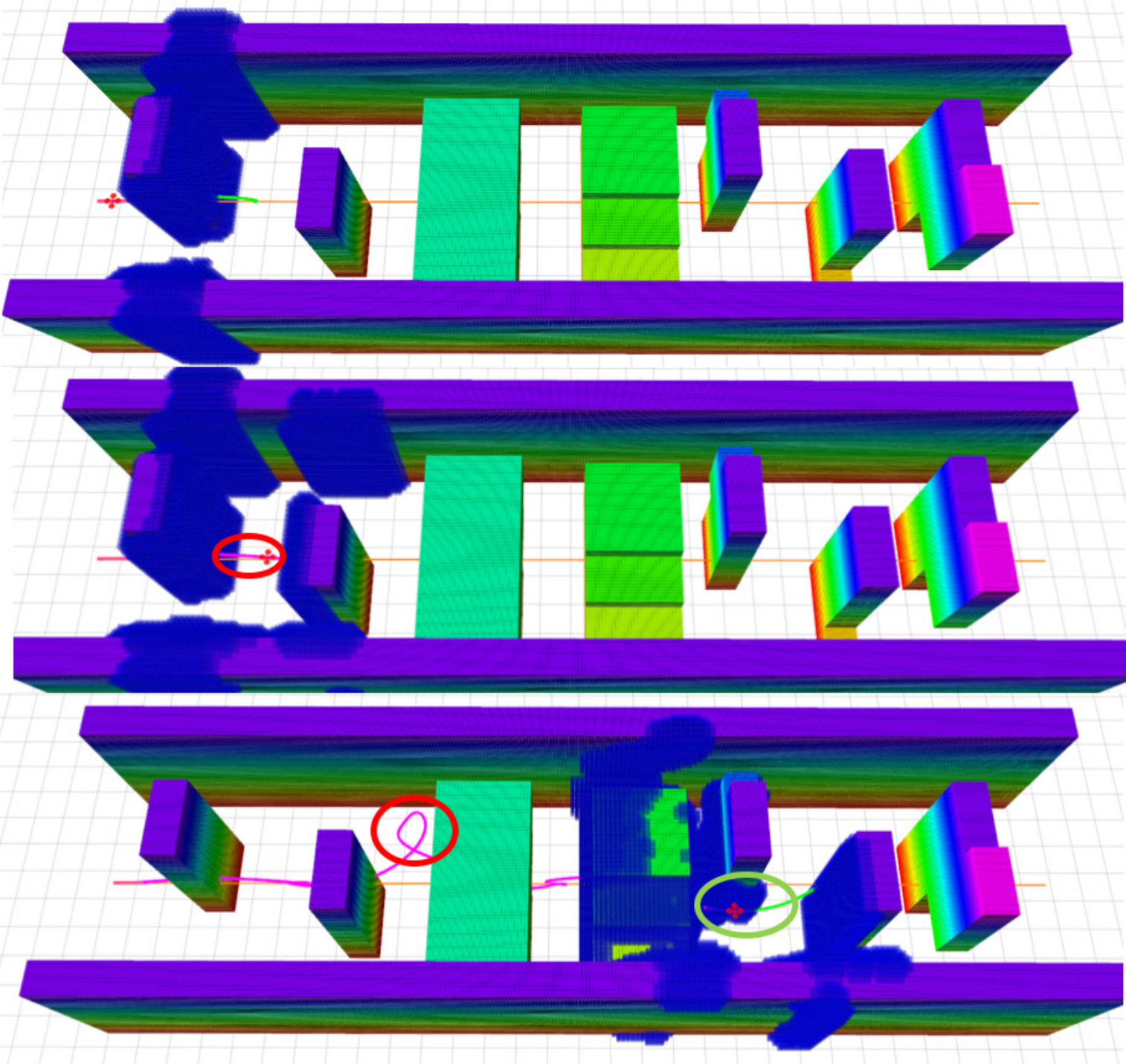}}
		\subfigure[SMP]{\includegraphics[trim={0cm 0cm 0cm 0cm},clip,width=0.24\textwidth]{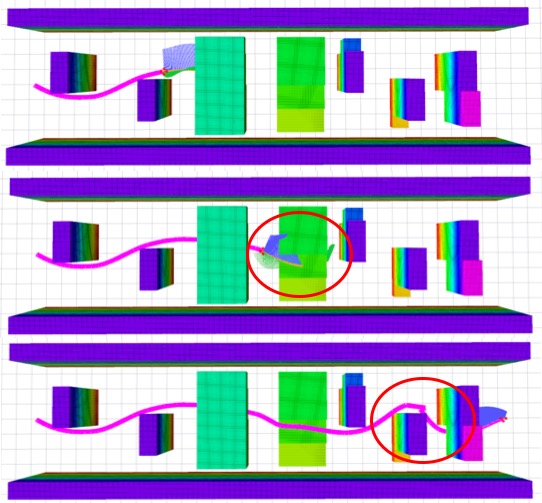}}
		\subfigure[SMP (Conservative)]{\includegraphics[trim={0.2cm 0cm 0.2cm 0.2cm},clip,width=0.226\textwidth]{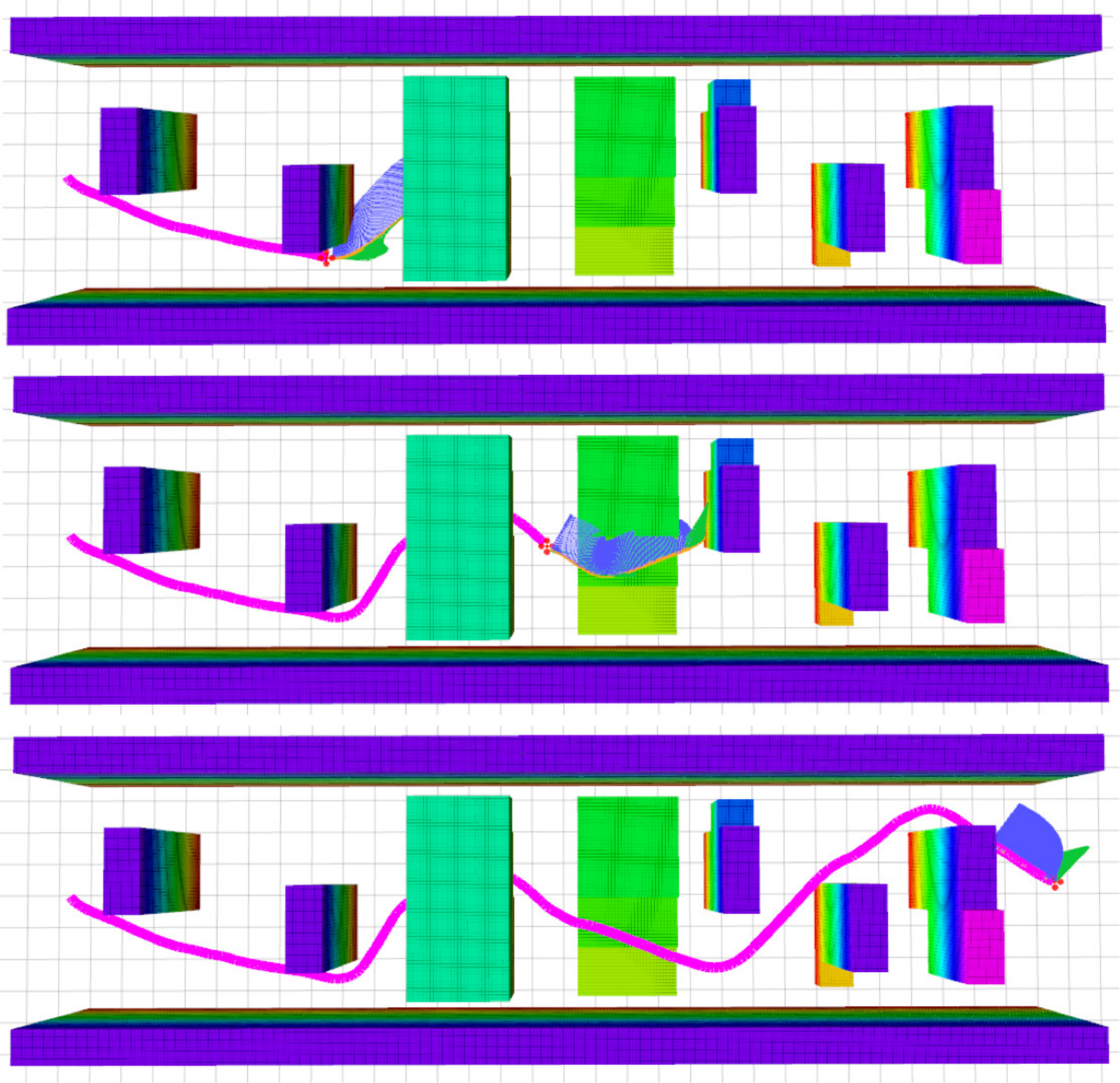}}
		\subfigure[Our Method]{\includegraphics[trim={0cm 0cm 0cm 0cm},clip,width=0.240\textwidth]{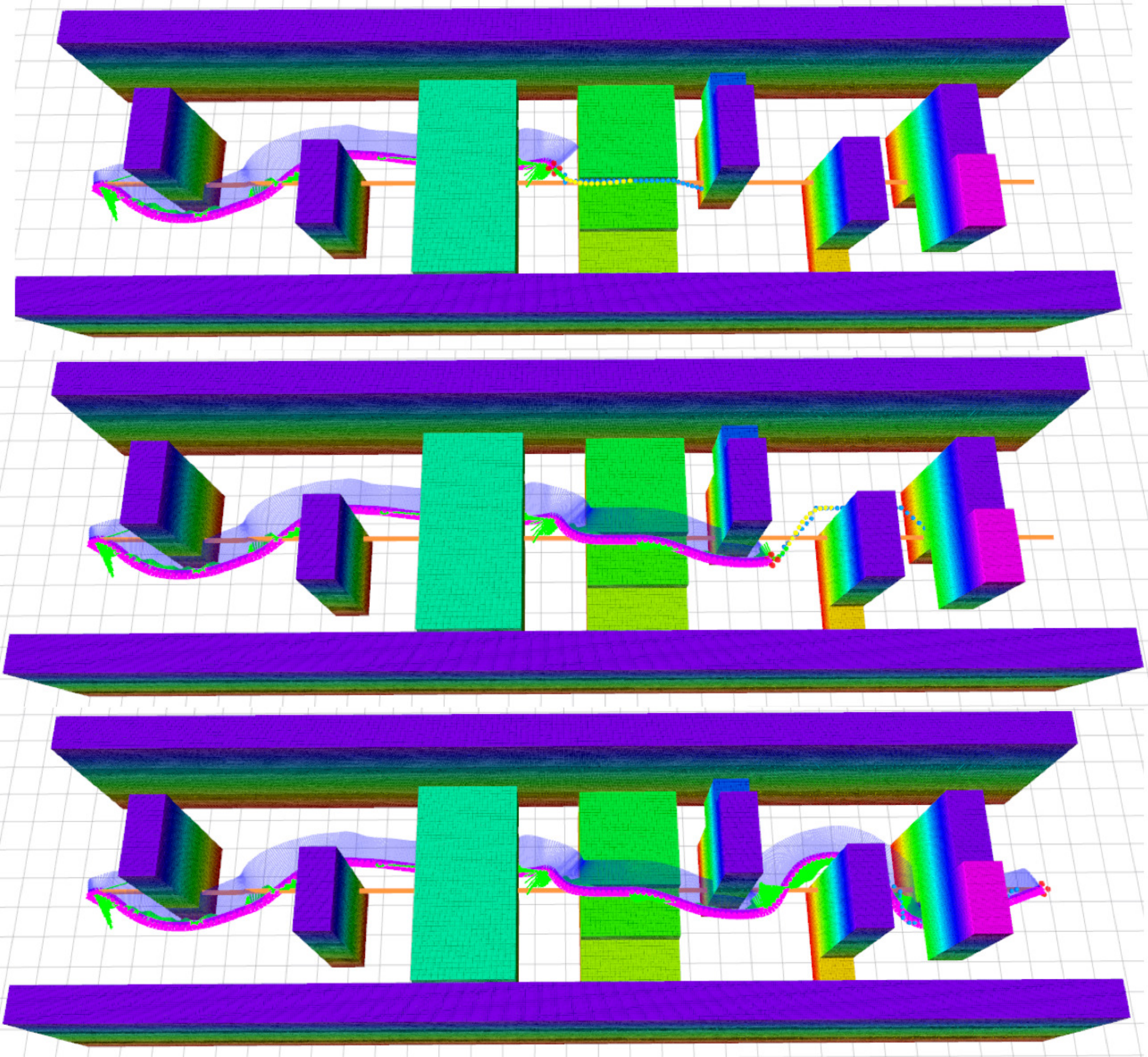}}
		\vspace{-0.3cm}
	\end{center}
	\caption{For the CT method in (a), ESDF is shown in blue. The red circle in (a) marks the part that is either infeasible or jerky due to the local minimum issue. For SMP with a full dynamic bound in (b), the primitive duration is set to $1s$ and level of discretization is set to 3. The red circle marks the part which does not satisfy the real-time requirement ($10$Hz). For SMP (Conservative) in (c), the maximum velocity is set to $2m/s$, the maximum acceleration is set $2m/s^2$, and the level of discretization is set to 1. For our method in (d), we use the default parameters in Sec.~\ref{sec:analysis}. }
	\label{fig:replan_tuple}
	\vspace{-1.0cm}
\end{figure*}

The RBK method can be done efficiently under full dynamic bounds, since grid-based B-spline search actually induces a discretization of the state space. The RBK search outputs a snap-continued trajectory, while the real-time SMP is only velocity-continued. The limitation of the RBK method is its greedy nature and limited representation of trajectories due to the discretization of the grid. The advantage of using RBK search is to obtain a dynamically-feasible time-parameterized trajectory efficiently. Thanks to this property, the post-optimization process has a feasible initial solution and rarely fails. Moreover, the EO refinement provides both safety and dynamical feasibility guarantee.

\begin{table}
	\centering
	\caption{Performance of Different Replanning Methods}
	\label{tab:simulation_replan}
	\resizebox{\columnwidth}{!}{
		\begin{tabular}{|c|c|c|c|c|c|}
			\hline
			\multirow{3}{*}{Method} & \multicolumn{2}{c}{\textbf{Trajectory Statistics}} & \multicolumn{3}{|c|}{\textbf{Time Efficiency}(s)} \\
			\cline{2-6}
			&  Mean		 & Max& 	\multirow{2}{*}{Ave} & \multirow{2}{*}{Max} & \multirow{2}{*}{Std} \\[-5pt]
			& Vel. ($m/s$) & Acc.($m/s^2$) &  &  &  \\
			\hline
			\multicolumn{1}{|c|}{SMP}   	 & 1.46  & 1.73  & 0.064 & 0.621  & 0.086 \\
			\hline
			\multicolumn{1}{|c|}{SMP (Conservative)  }  & 1.20  & 1.56  & 0.010 & 0.080  & 0.012 \\
			\hline
			\multicolumn{1}{|c|}{Our Method} & 1.17  & 2.69  & 0.035 & 0.064  & 0.012  \\
			\hline
		\end{tabular}
	}
	\vspace{-1.0cm}
\end{table}

\section{Experimental Results}
\label{sec:experimental}
For onboard testing, the parameters are as follows: the time step $\Delta t$ is set to $0.35s$; the maximum velocity and maximum acceleration are set to $1.2m/s$ and $2.0m/s^2$ repectively; and the local planning range is set to $10m \times 6m \times 1.1m$ (the corresponding gird size is $55 \times 35 \times 6$).

\subsection{Indoor Replanning Performance}
As shown in Fig.~\ref{fig:onboard_replan}, our replanning system works in complex 3-D environment with only local map. The whole trajectory length is $18.6m$ and total trajectory execution time is $43.4s$. The average velocity of the quadrotor is $0.45m/s$ with a maximum velocity of $0.79m/s$. The maximum acceleration of the trajectory is $0.58m/s^2$ and the whole trajectory is dynamically feasible. There are totally $125$ calls of RBK search (active mode) with average computation time $0.010s$, since the grid size is smaller compared to that in simulation. There are $105$ calls of EO. The average computation time of elastic tube expansion and  elastic optimization is $0.001s$ and $0.031s$, respectively.
\begin{figure}[htb]
		\vspace{-0.0cm}
	\begin{center}
\subfigure{\includegraphics[trim={0cm 0cm 0cm 0cm},clip, width=0.16\textwidth]{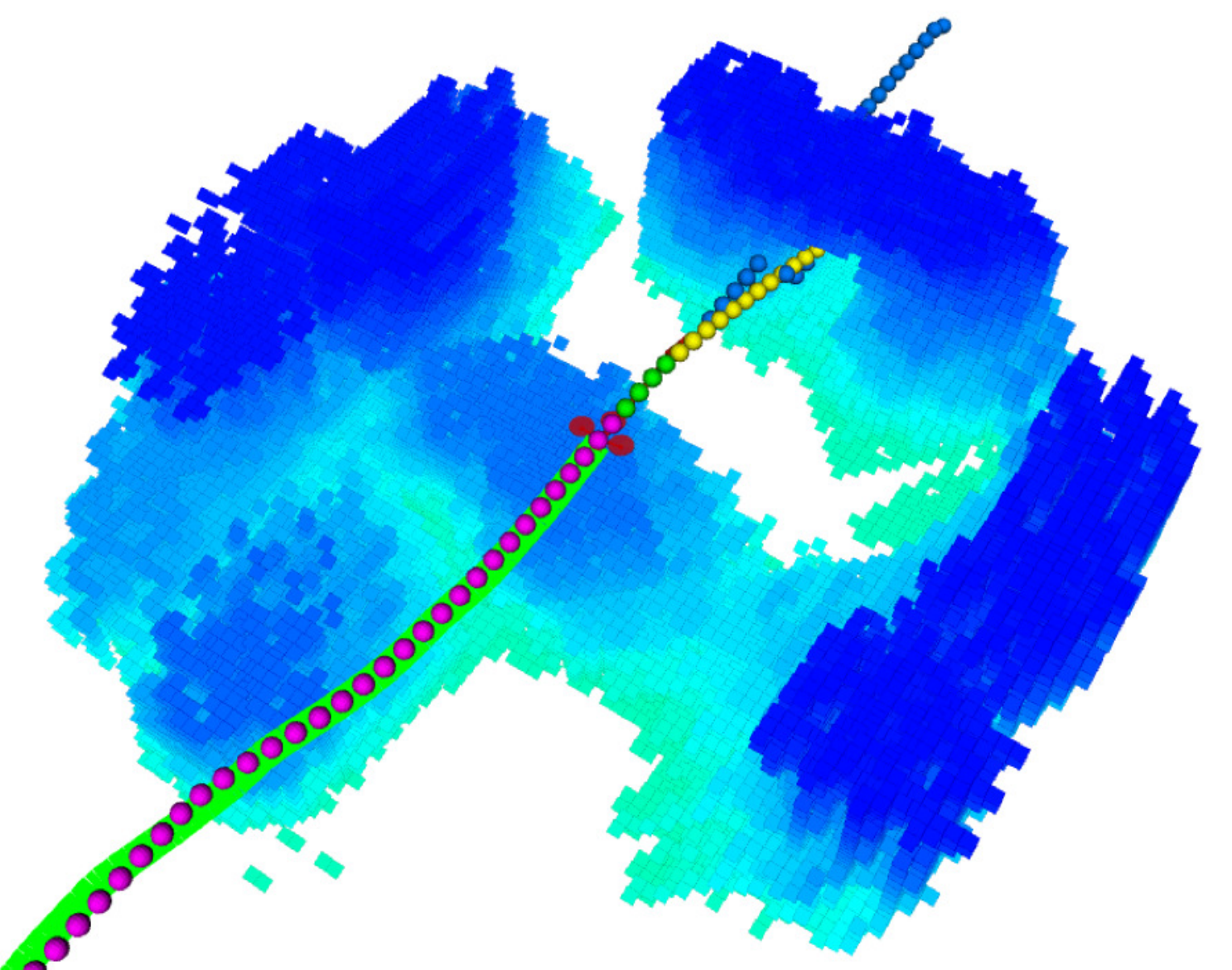}}
	\subfigure{\includegraphics[trim={0.7cm 0cm 0.5cm 0.0cm},clip, width=0.185\textwidth]{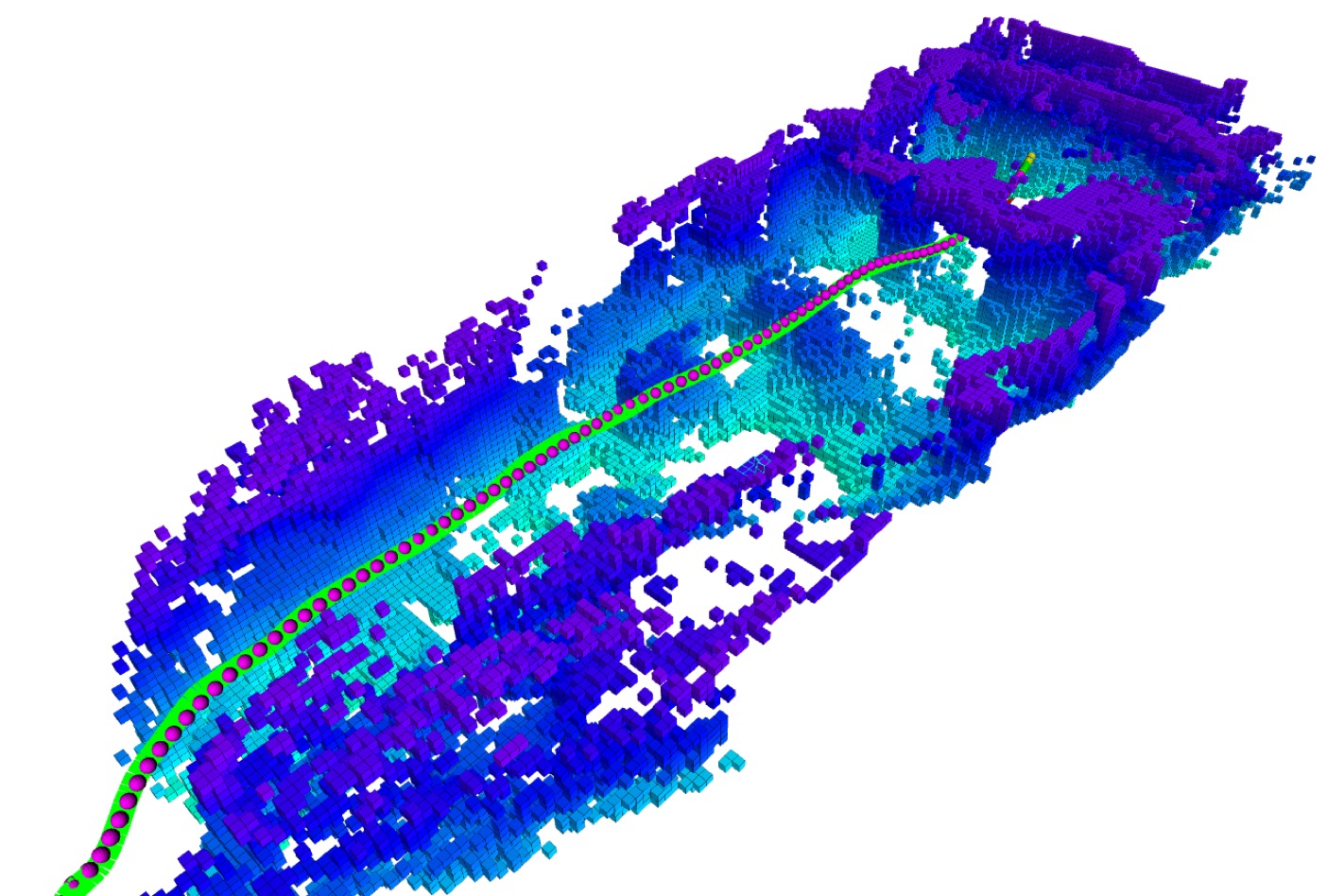}}
	\end{center}
	\caption{In (a), only local map is available for replanning. The control points found by RBK (in blue), executed control points (in pink), committed control points (in green), and optimizing control points (in yellow) are marked. The whole trajectory is shown in (b) and the trajectory following result is shown in green.}
	\label{fig:onboard_replan}
	\vspace{-0.6cm}
\end{figure}

\subsection{Outdoor Replanning Performance}
As shown in Fig.~\ref{fig:onboard_plan_outdoor}, we demonstrate the outdoor experiments. For Fig.~\ref{fig:onboard_plan_outdoor}~(a), the trajectory length is $19.6m$ and total execution time is $41.3s$. The average velocity of the quadrotor is $0.49m/s$. The maximum acceleration of the trajectory is $1.06m/s^2$ and the whole trajectory is dynamically feasible. There are totally $35$ calls of RBK search (passive mode) with average computation time $0.025s$. The average computation time of elastic tube expansion and  elastic optimization is $0.001s$ and $0.030s$, respectively. For Fig.~\ref{fig:onboard_plan_outdoor}~(b), the performance is similar and the detailed statistics are shown in the video attachment.
\begin{figure}[htb]
	\vspace{-0.0cm}
	\begin{center}
	\subfigure[]{\includegraphics[trim={0cm 0cm 0cm 0cm},clip,width=0.16\textwidth]{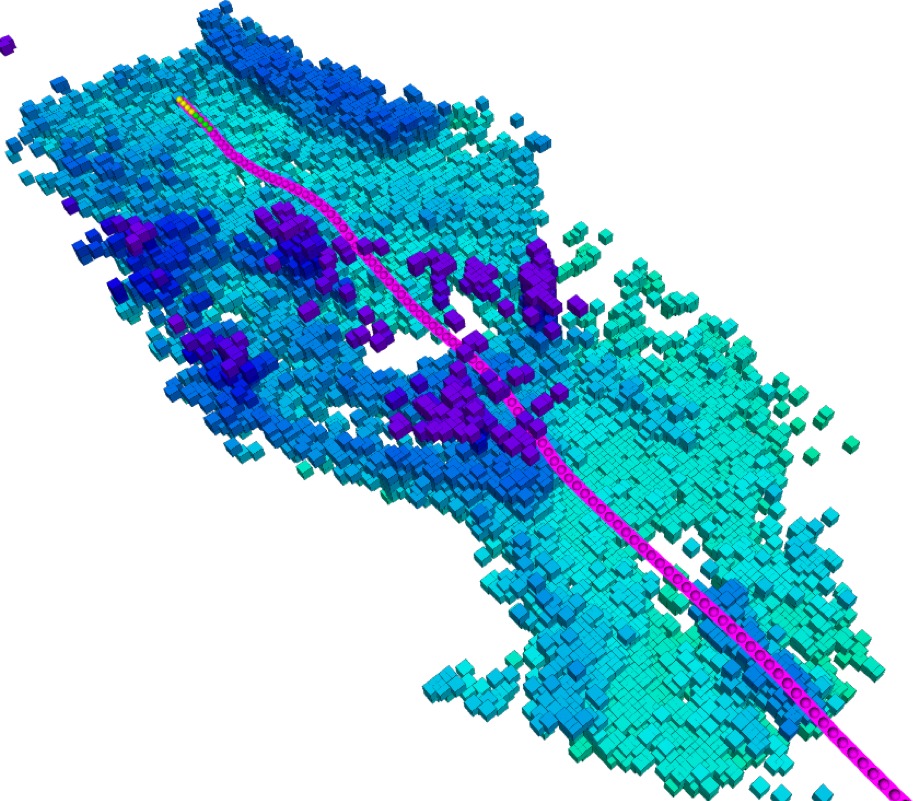}}
	\subfigure[]{\includegraphics[trim={0cm 0cm 0cm 0cm},clip,width=0.15\textwidth]{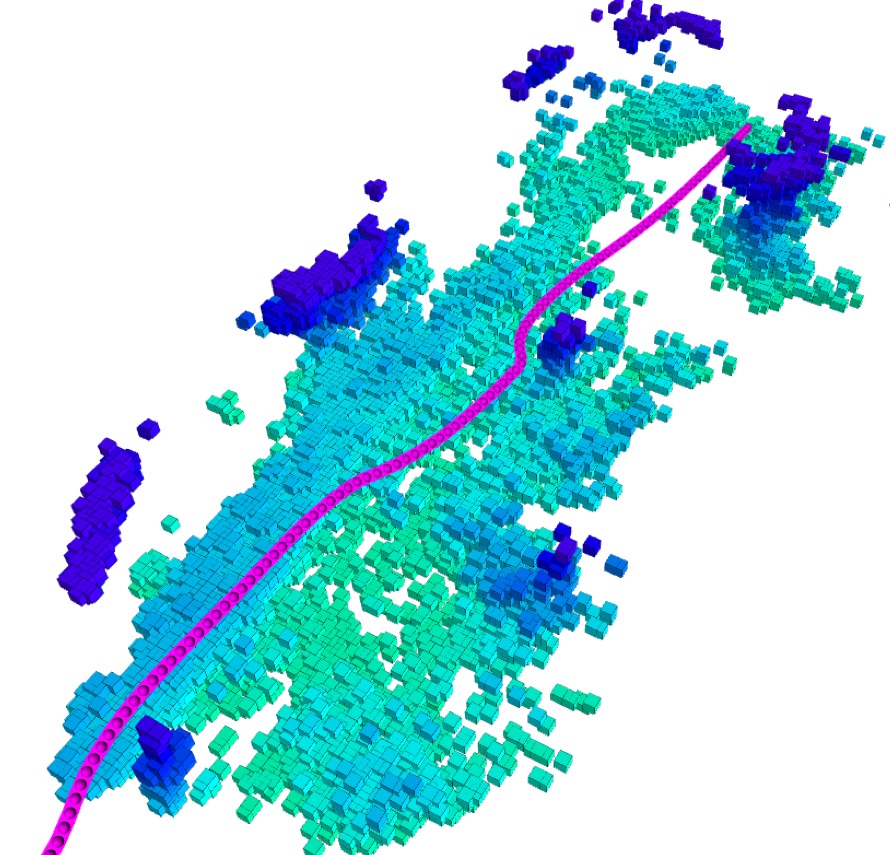}}
	\end{center}
	\vspace{-0.3cm}
	\caption{Illustration of the outdoor experiments. (a) shows a flight through a pavilion (part of the map on the top is cut for visualization purpose). (b) shows the flight avoiding trees.}
	\label{fig:onboard_plan_outdoor}
	\vspace{-0.4cm}
\end{figure}

\section{Conclusion and Future Work}
\label{sec:conclusion}
We focus on the replanning scenario for quadrotors. We present the RBK search method, which transforms the position-only search (such as A* and Dijkstra) into a kinodynamic search, by exploring the properties of B-spline. The RBK method facilitates the quadrotor's non-static initial state and produces dynamically feasible time-parameterized trajectory instead of unparameterized path. We present an EO approach as the post-optimization process to refine the trajectory, while the safety and dynamical feasibility are guaranteed. The dynamical feasibility is guaranteed for both the front-end and back-end to ensure the robustness of the optimization process. We design a receding horizon replanner using sliding window optimization strategy and local control property of B-spline. System comparisons and onboard experiments are provided to validate the superior performance of our replanning system. Extensions to this work may include developing informative heuristics for the RBK search to enhance the performance.

\bibliography{paper}
\end{document}